\documentclass{article}

\usepackage{microtype}
\usepackage{graphicx}
\usepackage{subcaption}
\usepackage{cancel}
\usepackage{booktabs} % for professional tables
\usepackage{tikz}
\usepackage{soul}
\usepackage{thm-restate}
\usetikzlibrary{decorations.pathmorphing}

\usepackage{hyperref}

\usepackage[accepted]{icml2021}

\icmltitlerunning{Fairness for Image Generation with Uncertain Sensitive Attributes}

\usepackage{amsmath}
\usepackage{amsthm}
\usepackage{amssymb}
\usepackage{multirow}
\usepackage{tikz}
\usepackage{adjustbox}
\usepackage{etoolbox}

\makeatletter

\newcommand{\define}[4][ignore]{%
  \ifstrequal{#1}{ignore}{}{
  \@namedef{thmtitle@#2}{#1}}%
  \@namedef{thm@#2}{#4}%
  \@namedef{thmtypen@#2}{lemma}%
  \newtheorem{thmtype@#2}[theorem]{#3}%
  \newtheorem*{thmtypealt@#2}{#3~\ref{#2}}%
}

\newcommand{\state}[1]{%
  \@namedef{curthm}{#1}
  \@ifundefined{thmtitle@#1}{
  \begin{thmtype@#1}
    }{
  \begin{thmtype@#1}[\@nameuse{thmtitle@#1}]
  }
    \label{#1}
    \@nameuse{thm@#1}
  \end{thmtype@#1}
  \@ifundefined{thmdone@#1}{
  \@namedef{thmdone@#1}{stated}%
  }{}
}

\newcommand{\restate}[1]{%
  \@namedef{curthm}{#1}
  \@ifundefined{thmtitle@#1}{
    \begin{thmtypealt@#1}
    }{
  \begin{thmtypealt@#1}[\@nameuse{thmtitle@#1}]
  }
    \@nameuse{thm@#1}
  \end{thmtypealt@#1}
  \@ifundefined{thmdone@#1}{
  \@namedef{thmdone@#1}{stated}%
  }{}
}

\newcommand{\thmlabel}[1]{
  \@ifundefined{thmdone@\@nameuse{curthm}}{\label{#1}
    }{\tag*{\eqref{#1}}}
}
\makeatother
\usetikzlibrary{positioning, calc, chains}
\newtheorem{theorem}{Theorem}[section]
\newtheorem{lemma}[theorem]{Lemma}
\newtheorem{example}[theorem]{Example}
\newtheorem{definition}[theorem]{Definition}

\newcommand{\norm}[1]{\|#1\|}

\newcommand{\wh}{\widehat}

\newcommand{\R}{\mathbb{R}}

\newcommand{\RN}[1]{%
  \textup{\uppercase\expandafter{\romannumeral#1}}%
}

\newcommand{\vertiii}[1]{{\left\vert\kern-0.25ex\left\vert\kern-0.25ex\left\vert #1 
		\right\vert\kern-0.25ex\right\vert\kern-0.25ex\right\vert}}

\DeclareMathOperator*{\E}{\mathbb{E}}

\newcommand{\cA}{\mathcal A}
\newcommand{\cB}{\mathcal B}

\newcommand{\cN}{\mathcal N}

\usepackage{lineno}

\usepackage{bbm}
\usepackage{bm}
\usepackage{comment}
\usepackage{cleveref}
\allowdisplaybreaks

\begin{document}

\twocolumn[
\icmltitle{Fairness for Image Generation with Uncertain Sensitive
Attributes}

\icmlsetsymbol{equal}{*}

\begin{icmlauthorlist}
\icmlauthor{Ajil Jalal}{equal,ece}
\icmlauthor{Sushrut Karmalkar}{equal,cs}
\icmlauthor{Jessica Hoffmann}{equal,cs}
\icmlauthor{Alexandros G. Dimakis}{ece}
\icmlauthor{Eric Price}{cs}
\end{icmlauthorlist}

\icmlaffiliation{cs}{Department of Computer Science, The University of
Texas at Austin}
\icmlaffiliation{ece}{Department of Electrical and Computer
Engineering, The University of Texas at Austin}

\icmlcorrespondingauthor{Ajil Jalal}{ajiljalal@utexas.edu}
% \icmlcorrespondingauthor{Sushrut Karmalkar}{sushrutk@cs.utexas.edu}
% \icmlcorrespondingauthor{Jess Hoffmann}{hoffmann@cs.utexas.edu}
% \icmlcorrespondingauthor{Alex Dimakis}{dimakis@austin.utexas.edu}
% \icmlcorrespondingauthor{Eric Price}{ecprice@cs.utexas.edu}

\icmlkeywords{Algorithmic Fairness, Image Super-resolution, Generative
Priors, Deep Generative Models, Langevin Dynamics, Posterior Sampling}

\vskip 0.3in
]



\printAffiliationsAndNotice{\small \icmlEqualContribution. Our code
and models are available at:
\url{https://github.com/ajiljalal/code-cs-fairness}.}

\begin{abstract}
  This work tackles the issue of fairness in the context of generative
  procedures, such as image super-resolution, which entail different
  definitions from the standard classification setting. Moreover,
  while traditional group fairness definitions are typically defined
  with respect to specified protected groups -- camouflaging the
  fact that these groupings are artificial and carry historical and
  political motivations -- we emphasize that there are no  ground
  truth identities. For instance, should South and East Asians be
  viewed as a single group or separate groups?  Should we consider one
  race as a whole or further split by gender?  Choosing which groups
  are valid and who belongs in them is an impossible dilemma and being
  ``fair'' with respect to Asians may require being ``unfair'' with
  respect to South Asians.  This motivates the introduction of
  definitions that allow algorithms to be \emph{oblivious} to the
  relevant groupings. 

  We define several intuitive notions of group fairness and study
  their incompatibilities and trade-offs. We show that the natural
  extension of demographic parity is strongly dependent on the
  grouping, and \emph{impossible} to achieve obliviously.  On the
  other hand, the conceptually new definition we introduce,
  Conditional Proportional Representation, can be achieved obliviously
  through Posterior Sampling.  Our experiments validate our
  theoretical results and achieve fair image reconstruction using
  state-of-the-art generative models.
\end{abstract}

\section{Introduction}

\begin{figure*}[t]
	\centering
		\includegraphics[width=0.9\textwidth]{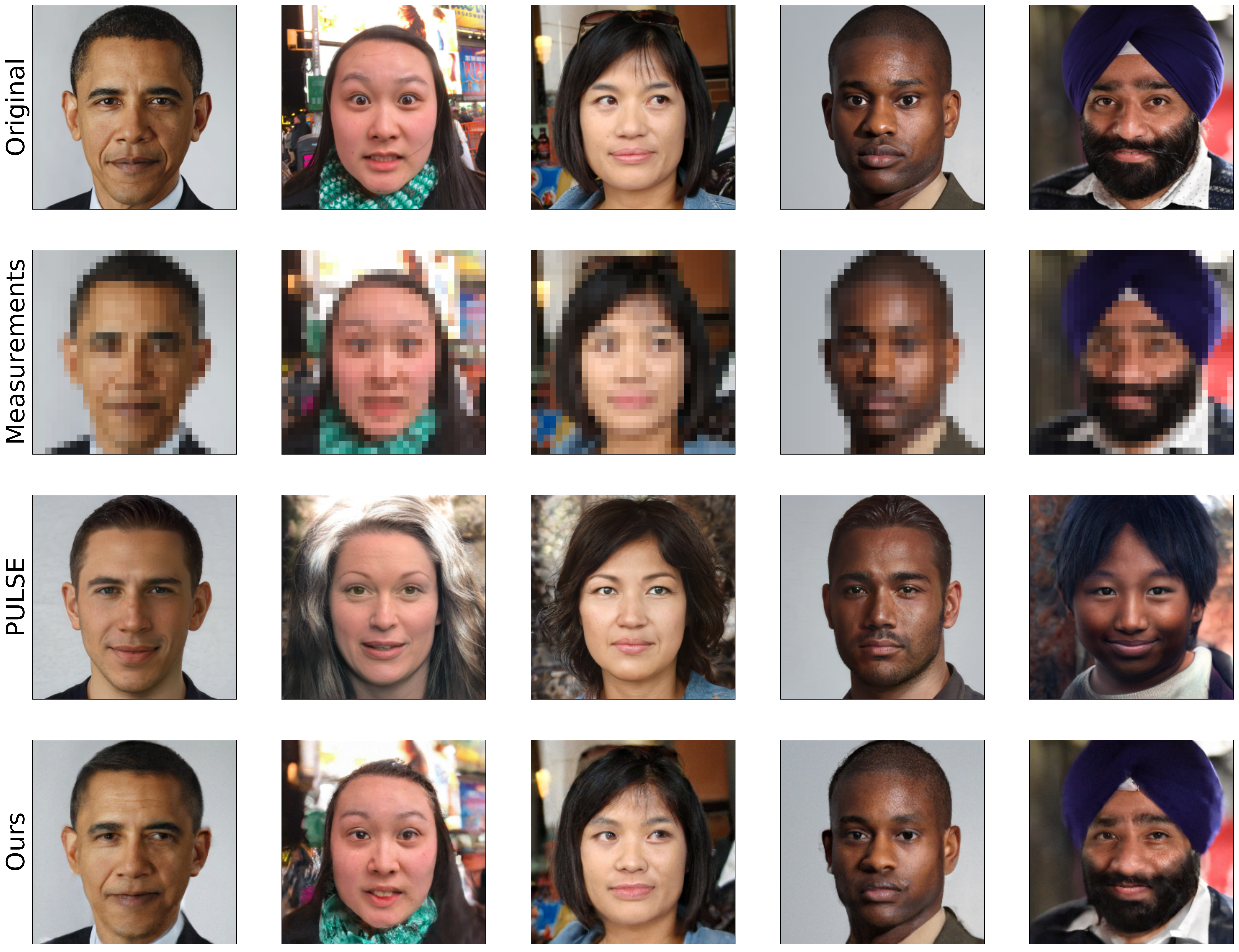}
		\vspace*{-3mm}
    \caption{Super-resolution reconstructions on Barack Obama and four faces from the FFHQ dataset. The top row shows
    original images, the second row shows what the algorithms observe: blurry measurements after downsampling by $32\times$ in each dimension. The third row shows reconstructions by PULSE, and the last row shows reconstructions by Posterior Sampling via Langevin dynamics, the algorithm we are advocating for.
    These faces were chosen to compare performance on various ethnicities. Please see Appendix~\ref{app:ffhq} for images chosen at random from the dataset.}
    \label{fig:reconstr-ffhq}
\end{figure*}

Fairness, accountability, and transparency have taken a front-row seat
in the machine learning community. Numerous recent controversies have
erupted over how current machine learning systems already in use can
be  racist \cite{simonite2018comes}, sexist \cite{kay2015unequal},
homophobic \cite{morse2017google}, or all of the above
\cite{moore2016microsoft}. In a recent controversy, a low-resolution
image of Barack Obama was put into PULSE, a super-resolution
generative model~\cite{pulse}, but the resulting image was of a
distinctly White man.  While we generally have to be careful when
identifying the race of a person that does not exist, such as the one
represented by the generated image, multiple other reconstructions by
PULSE strongly suggest that this algorithm contributes to the systemic
bias against people of color.

Accuracy of representation as a fairness notion is a significant leap
from the more traditional classification setting, in which we require
some form of independence (or conditional independence) between the
sensitive attributes and the algorithm prediction.  In the context of
image reconstruction, the output itself can be considered as having
sensitive attributes, and we want the sensitive attributes of the
input to match the sensitive attributes of the output -- which is
fundamentally different from an independence condition. This leads us
to introduce and discuss new fairness definitions, specific to the
field of image generation, reconstruction, denoising and
super-resolution.

In light of the ``White Obama" controversy~\cite{pulse}, it has been
suggested that reconstruction algorithms are biased because the
datasets are not representative of the true population distribution.
While it is true that the datasets are biased
\cite{buolamwini2018,khosla2012undoing}, current algorithms also play
their part in widening this gap
\cite{wang2019balanced,terhorst2020face}, such that majority classes
get overrepresented, and minorities get further underrepresented.
Indeed, when applying PULSE~\cite{pulse} to an unbalanced dataset with
80\% dogs (majority class) and 20\% cats, we observe that 80\% of cats
are mistakenly reconstructed as dogs, while only 2\% of dogs are
reconstructed as cats (see Figure \ref{fig:cat20dog80}).   When cats
are the 80\% majority, the situation reverses to 1\% and 98\%
mistakes, respectively (see Figure~\ref{fig:cat80dog20}).

There is a simple intuitive reason why reconstruction algorithms
designed to maximize accuracy will increase bias. Assume we observe a
noisy version $y$ of an image $x^*$ that is either a dog or a cat.
Assume cats are the minority, with the prior  $\Pr(x^* \in
\text{Dog})=0.8$. Further, assume that the measurements are always
noisy and cannot definitively identify the species, so cat-like
measurements are such that $p( y \mid x^* \in \text{Cat})/p( y \mid
x^* \in \text{Dog}) \leq 2$. Using Bayes, the posterior is
\begin{align*} 
\Pr(x^* \in \text{Dog} \mid y) &= p(y \mid x^* \in \text{Dog}) \cdot \frac{\Pr(x^* \in \text{Dog})}{p(y)} \\
&\geq 1 \cdot \frac{0.8}{0.8 \cdot 1 + 0.2 \cdot 2} \\
&= 2/3.
\end{align*}
Therefore, regardless of the measurement, an algorithm that maximizes
accuracy \textit{will always produce images of dogs}.

This issue relates to a rich area of work on fairness in machine
learning, including for classification or generation without
measurements (see Section~\ref{sec:relatedwork} for an overview).
However, to the best of our knowledge, previous approaches always
assume that the sensitive attributes are well-defined and unambiguous.
While this assumption might hold for cats and dogs, as
\cite{benthall2019racial,hanna2020towards} emphasize, race cannot be
treated in the same way. First, it is unclear when to include
subgroups within the larger group or when to treat them separately
(for instance, when to consider South Asians as their own subgroup, or
as Asians). This has major implications, as choosing which groups
exist and what sensitive attributes are valid can already widen
existing discrimination, as the long line of research on
intersectionality shows. Second, even if we could decide on which
groups are relevant, races are multidimensional and cannot be reduced
to a simple categorical value: studies show that we can arrive at
inconsistent conclusions about the same data depending on how race is
measured (e.g. self-reported or observed) \cite{howell2017}. Our work
therefore focuses on moving away from classifying people into
partitions.

\paragraph{Problem Setting.}  Suppose that we have a distribution of
users $x^*$; each user $x^*$ is observed through some lossy
observation process to produce $y$ (e.g., a low-resolution image); and
our reconstruction algorithm produces $\wh{x}$ from $y$.  We are
concerned about fairness with respect to a collection of protected
groups  $C = \{c_1, \dotsc, c_k\}$.  Our setting therefore includes,
but is not limited to, the special case in which $C$ is a
partition\footnote{For simplicity of notation, each group $c_i$
contains both \emph{people} $x^*$ and \emph{images} $\wh{x}$.}.

The fairness concern we consider is that of \emph{representation}:
when users in each protected group use the algorithm, does the result
adequately represent them and their group?  When the observation
process $y$ is significantly lossy, there inevitably will be
``representation errors'' where a member of one group is reconstructed
as being in a different group.  How should we determine if the errors
are equitable?

\paragraph{Our Contributions: Fairness Definitions.}

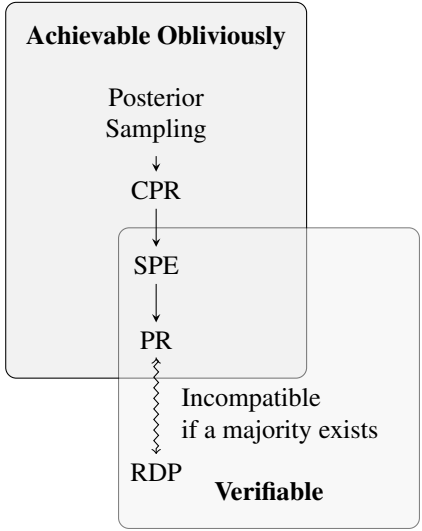
\begin{figure}
    \centering
    \begin{tikzpicture}
\draw[fill=gray!10, rounded corners] (-2,-2.5) rectangle (2, 2.5);
\draw[fill=gray!10, opacity=0.5, rounded corners] (-0.5,-4.5) rectangle (3.5, -0.5);

\node (PS) at (0,1) {\begin{tabular}{c}
Posterior \\ Sampling \end{tabular}};     
\node (cpr) at (0,0) {CPR};
\node (spe) at (0,-1) {SPE};
\node (pr) at (0, -2) {PR};
\node (rdp) at (0, -3.75) {RDP};

\filldraw[black] (0, 1.75) circle (0pt) node[anchor=south] {\textbf{Achievable Obliviously}};
\filldraw[black] (1.5, -4.25) circle (0pt) node[anchor=south] {\textbf{Verifiable}};
\draw[->, >=stealth] (PS) -- (cpr);
\draw[->, >=stealth] (spe) -- (pr);
\draw[->, >=stealth] (cpr) -- (spe);
\draw[<->, decorate, decoration={zigzag,
    segment length=4,
    amplitude=.9,post=lineto,
    post length=2pt }]  (pr) -- (rdp); 
\filldraw[black] (0, -3.0) circle (0pt) node[anchor=west] {\begin{tabular}{l}Incompatible\\if a majority exists\end{tabular}};

\end{tikzpicture}
    \caption{CPR, SPE and PR are achievable obliviously by Posterior
    Sampling. However, RDP cannot be achieved obliviously, and if a majority group exists it cannot be achieved simultaneously with PR.} 
    \label{fig:relationships_properties}
\end{figure}

We introduce definitions for some natural notions of fairness in
reconstruction.  One is that the average representation rate should be
independent of the group: 
\begin{align}
  \Pr(\wh{x} \in c_i \mid x^* \in c_i) \tag{RDP}
\end{align}
is the same value for all $i \in [k]$.  We call this
\emph{Representation Demographic Parity} (RDP), by analogy to the
binary classification setting, where Demographic Parity means that
$\Pr(L=1 \mid x^* \in c_i)$ is fixed. The difference here is that the
``good'' outcome $(\wh{x} \in c_i)$ is different for each group, while
typically in classification the ``good'' outcome (where, e.g., $L = 1$
means ``offer a loan'') is the same across groups.  RDP is simply
requesting that the reconstructions have the same error rates across
groups.

An alternative definition is that the demographics of the output
should match those of the input:
\begin{align}
  \Pr(\wh{x} \in c_i) = \Pr(x^* \in c_i)~~\forall i.\tag{PR}
\end{align}
We call this \emph{Proportional Representation} (PR).  It simply says
that the reconstruction process should not introduce bias in the
distribution for or against any group.

Unfortunately, these two definitions are often \emph{incompatible}.
We show in Proposition~\ref{prop:rdp_pr_incompat} that, whenever a
majority group exists and the measurements can confuse it with other
groups, no algorithm can achieve both RDP and PR.

One weakness of both PR and RDP is that they only consider the
\emph{global} behavior of the reconstruction.  But individual users
want to be represented well when they use the system, and may not be
mollified by the knowledge that many other members of their group are
being represented.  On the other hand, some  images are genuinely
harder to reconstruct accurately, so expecting equal representation
accuracy/RDP for every user would strongly limit overall accuracy.
Our solution is to extend PR by incorporating the measurement process:
\begin{align}
  \Pr(\wh{x} \in c_i \mid y) = \Pr(x^* \in c_i \mid y)~~\forall i, y. \tag{CPR}
\end{align}
We call this \emph{Conditional Proportional Representation} (CPR).
The idea is that the population of users with each given $y$ should
have fair treatment (in the sense of PR).  Of course, CPR implies PR
by averaging over $y$.

Note that CPR implies that the reconstruction process must be
\emph{randomized}, not deterministic.  This has other benefits: if the
user is not satisfied with the result, they can rerun the algorithm
until they get a result that represents them.  Users can also get a
collection of $\wh{x}_i$ to observe the diversity of possible
reconstructions.

\begin{figure*}[t]
	\centering
		\includegraphics[width=0.98\textwidth]{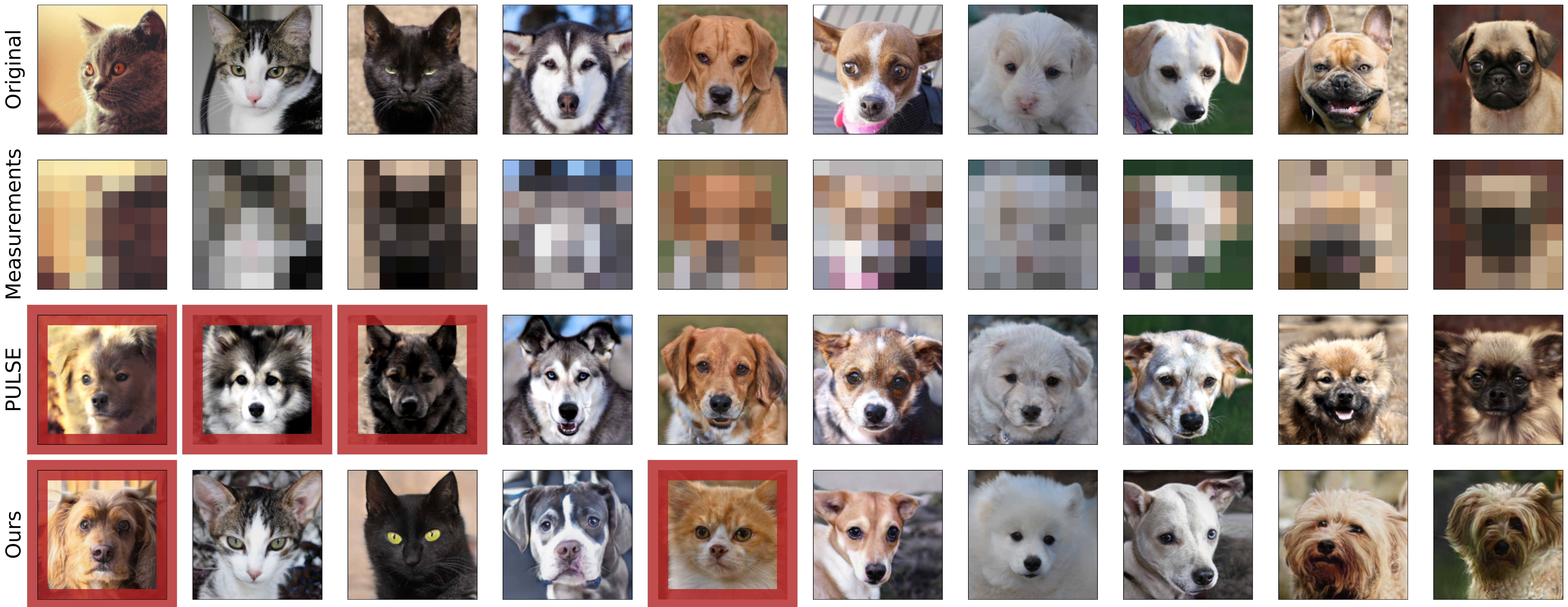}
		\vspace*{-3mm}
		\caption{Super-resolution on the AFHQ cats \& dogs dataset using
		StyleGAN2 trained on \textbf{20\% cats} and \textbf{80\% dogs}.
		The rows show, from top to bottom 1)  original images 2)
		measurements after downsampling $64\times$ in each dimension 3)
		reconstructions by PULSE 4) reconstructions by Posterior Sampling.
		The red bounding boxes denote the errors.  PULSE converts almost
		all cats to dogs, and almost never does the reverse. Posterior
		Sampling makes roughly the same number of errors on cats and
		dogs.}
		\label{fig:cat20dog80-examples}
\end{figure*}

\paragraph{Our Contributions: Algorithms.}
We show that CPR (and hence PR) are achievable with a
simple-to-describe algorithm: posterior sampling, where we output
$\wh{x} \sim p(x^* \mid y)$.  This can be approximated well in
practice using Langevin dynamics for state-of-the-art generative
models representing $p(x^*)$, as we discuss in
Section~\ref{sec:langevin}.

Posterior Sampling also satisfies one more fairness condition: the
confusion matrix is symmetric, meaning that (for example) an equal
number of Black users will be reconstructed as White as White users
will be reconstructed as Black.  We call this condition
\emph{Symmetric Pairwise Error} (SPE).  
\begin{align}
	\Pr\left( \wh{x} \in c_i , x^* \in c_j \right) = \Pr\left( \wh{x}
	\in c_j, x^* \in c_i \right). \tag{SPE}
\end{align}
for all $i,j \in [k]$.  CPR implies SPE, and SPE implies PR (see
Figure~\ref{fig:relationships_properties}).

Since SPE implies PR, in general SPE is incompatible with RDP (per
Proposition~\ref{prop:rdp_pr_incompat}).  But in the special case of
two groups $c_1, c_2$ of equal size, then SPE actually implies RDP.
This gives an algorithm to achieve RDP for the two-group setting: we
reweight our input distribution such that each group has equal
probability, then perform Posterior Sampling with respect to the
reweighted distribution.  With more than two groups, there still
exists a reweighting of the groups such that Posterior Sampling on the
reweighted distribution satisfies RDP (see
Theorem~\ref{thm:cond_resamp_reweight}).  This reweighted-resampling
algorithm can be performed in practice by learning a GAN for the
reweighted distribution and using Langevin dynamics on the reweighted
GAN.

\paragraph{Our Contributions: Obliviousness.}
Posterior Sampling satisfies the CPR, PR, and SPE fairness criteria
while retaining an invaluable property: the algorithm doesn't depend
on the set of protected groups $C$.  It satisfies the fairness
properties for every set of protected groups,  which is an
algorithmically achievable way of addressing the issues raised
in~\cite{hanna2020towards} about race being ambiguous and ill-defined.
We say such an algorithm is \emph{obliviously} fair.  By contrast, our
reweighted-resampling algorithm achieving RDP needs to know the
protected groups, and would not satisfy RDP for a different collection
of groups.  Which fairness properties can be achieved obliviously, and
under what circumstances?

Our main results here are twofold: first, Posterior Sampling is the
\emph{only} algorithm that satisfies CPR obliviously.  Second, RDP
\emph{cannot} be satisfied obliviously.  This impossibility applies
even to obliviousness with respect to one of two plausible, socially
meaningful partitions.  Theorem~\ref{thm:white-asian} shows, for
example, that you cannot satisfy RDP with respect to both \{White,
Asian\} and \{White, South Asian, East Asian\} if your observations
are lossy.  This means that every algorithm can reasonably be viewed
as unfair with respect to RDP.

\paragraph{Our Contributions: Experiments.} We implement Posterior
Sampling via Langevin dynamics, study its empirical performance and
compare it to PULSE with respect to our defined metrics. We do this on
the MNIST~\cite{lecun1998mnist}, FlickrFaces-HQ~\cite{karras2019style}
and AFHQ cat \& dog~\cite{choi2020starganv2} datasets.  We evaluate
obliviousness and SPE of Posterior Sampling on the first two datasets.
Using the AFHQ cat \& dog dataset, we demonstrate empirically that
Posterior Sampling satisfies SPE and PR over various imbalances
between cats and dogs.

\subsection{Related Work} \label{sec:relatedwork}

Numerous works have attempted to tackle the issue of bias in the
machine learning of images, either for data generation/reconstruction
tasks or for downstream tasks such as face recognition and image
quality assessment. One popular approach for dealing with bias
consists of adversarially generating data or embeddings with a
discriminator for different values of the sensitive attributes,
yielding similar distributions for different values of the sensitive
attribute
\cite{madras2018learning,xu2018fairgan,xu2019fairgan+,gong2020jointly,khajehnej2020adadversarial,sattigeri2018fairness,yu2020inclusive}.
Another approach focuses on learning explicitly the bias of the
dataset, so as to remove it
\cite{khosla2012undoing,grover2019fair,choi2020fair}. The special case
of fair dimensionality reduction through principal component analysis
is solved by \cite{samadi2018price}. Another research direction
formulates the fairness constraints as an additional term in the loss
\cite{serna2020sensitiveloss}. Another approach focuses on minorities
and learns their specific features
\cite{amini2019uncovering,gong2020mitigating}.  A related line of work
improves the fairness of generative models without
retraining~\cite{tan2020improving}, however we do not know how to use
these for inverse problems.

Another relevant line of research  studies fairness in the presence of
uncertainty, either in the labels
\cite{kleinberg2018selection,blum2019recovering,wang2020fair,rolf2020balancing}
or in the sensitive attributes
\cite{awasthi2020equalized,lamy2019noise,celis2020fair,wang2020robust}.
In particular, one work studies overlapping groups
\cite{yang2020fairness}.

Super resolution using deep learning has had remarkable success at
producing accurate images. In \cite{ledig2017photo}, the authors
provide an algorithm which performs photo-realistic super resolution
using GANs. However, this model requires retraining of the GAN when
the measurement operator changes.  Subsequent work has overcome this
hurdle. Some models independent of the forward operator include
CSGM~\cite{bora2017compressed}, OneNet~\cite{rick2017one},
PULSE~\cite{menon2020pulse}, Deep Image Prior~\cite{ulyanov2018deep}
and Deep Decoder~\cite{heckel2019deep}. Another line of work has shown
that Posterior Sampling using approximate deep generative priors is
instance-optimal for compressed
sensing~\cite{jalal2021instance}.

\section{Fairness definitions for image generation}

\subsection{Representation Demographic Parity} While multiple group
fairness definitions (demographic parity, equalized odds or
opportunity, calibration
etc.\cite{barocas2017fairness,hardt2016equality}) have been studied
and widely accepted in the context of classification, their extension
to the setting of image generation is not immediate. Here, we extend
demographic parity.

\begin{definition}\label{def:RDP}
Let $x^* \in \R^n$ denote the ground truth, and $P$ denote its
distribution. Let $y \in \R^m$ be some measurements of $x^*$. For a
collection $C = \{c_1, \dots, c_k\}$ of (potentially overlapping)
sets, an algorithm which reconstructs $x^*$ using $y$ satisfies
Representation Demographic Parity (RDP) if: $$\forall i, j \in [k],
\,\, \Pr(\hat{x} \in c_i | x^* \in c_i) = \Pr(\hat{x} \in c_j | x^*
\in c_j). $$
\end{definition}

\begin{example}
If $k=2$, $c_1$ being all women, $c_2$ being all non-women,
Representation Demographic Parity with respect to these two groups
implies that women are as likely to be reconstructed as women as
non-women are to be reconstructed as non-women.
\end{example}

\subsection{Limitations of traditional group fairness
definitions}\label{sec:criticalrace} Inspired by
\cite{hanna2020towards}, we note several reasons for having fairness
definitions that are more flexible with respect to the groups in the
collection or partition.

\textbf{Minorities are ill-defined:} What constitutes a minority? Are
South Asians their own subgroup, or are they assigned as Asians? The
list of accepted minorities is not only inconsistent across location
and purpose, but multiple levels of granularity could be equally
valid.  Similar concerns can be raised from the point of view of
intersectionality: we might both be interested in the discrimination
faced by all women, and all people of color, without wanting to erase
the singular discrimination faced by women of color
\cite{buolamwini2018}.

\textbf{Races are multi-dimensional:} As \cite{roth2016multiple}
argues, races are multi-dimensional, and these dimensions are all
relevant, albeit in different settings. For instance, voting patterns
are more accurately predicted based on self-identified race, while
observed race is more informative when dealing with discrimination.
These differences are not minor: as \cite{howell2017} shows, measuring
races in five different ways led to widely different interpretations
of the same data.

\textbf{Partitions reify the status quo:} According to
\cite{hanna2020towards}, widespread adoption of race categories
participates in erasing their historical and social context
\cite{duster2005race,smart2008standardization}, as well as
perpetuating the current system and creating new harm
\cite{kaufman1999inconsistencies,sewell2016racism}.

\textbf{Who chooses the partition:}  \cite{barabas2020studying} raises
concerns on who has the power to choose the partitions and what their
intentions were. Historically, such partitions have done significant
harm to the minorities they were supposed to protect
\cite{mills2014racial,hanna2020towards}.

In response to these critiques, we  study a novel property of fairness
definitions.

\begin{definition}[obliviously]
 We say an algorithm satisfies a group fairness definition
 \emph{obliviously} if the algorithm satisfies the fairness definition
 for any collection of sets and does not require knowledge of the
 collection of sets to perform reconstruction. 
\end{definition}

Satisfying a fairness definition obliviously  is one way of addressing
the issues above, as it is now satisfied for all groups at the same
time.  This requirement may nevertheless be too strong, since most
such groupings are not socially meaningful. This leads to more
restricted versions of obliviousness, ones that only hold for specific
sets of collections.  Unfortunately, RDP cannot be satisfied even with
only two socially meaningful partitions.

\define{thm:white-asian}{Theorem}{
 Let $A$ and $B$ be
  disjoint groups (e.g., Asian and White people), and let
  $A_1, A_2 \subset A$ be disjoint groups that cannot be perfectly
  distinguished from measurements only (e.g., South Asians and East
  Asians).  Then Representation Demographic Parity cannot be satisfied
  $\{\{A, B\}, \, \{A_1, A_2, B\}\}$-obliviously.
}
\state{thm:white-asian}

In the example stated in Theorem~\ref{thm:white-asian}, it is
impossible to be fair as defined by Representation Demographic Parity
with respect to White people, South Asians, East Asians, and
Asians as a whole. This holds even if we know exactly what the
measurement process is, the demographics, and what the relevant groups
are.  

We can state this more generally:
\define[Representation Demographic Parity cannot be satisfied
obliviously]{thm:rdp-general}{Theorem}
{
  The only way for an algorithm to satisfy Representation Demographic
  Parity obliviously is to achieve perfect reconstruction.
}

\state{thm:rdp-general}

\subsection{Conditional Proportional Representation}

An alternative fairness measure is that the distribution of the output
of the algorithm should match the demographics of the input to the
algorithm:
\begin{definition}[Proportional Representation]
	In the setting of Definition~\ref{def:RDP}, an algorithm
	satisfies Proportional Representation (PR) if:  
	\begin{align*}
		\forall i \in \left[ k \right], \; \Pr\left( \wh{x} \in c_i \right) = \Pr\left( x^* \in c_i \right).
	\end{align*}
\end{definition}

One could also demand a much stricter fairness property, where the
algorithm should satisfy PR among the population that maps to the same
observation, for every possible observation:

\begin{definition}[Conditional Proportional Representation]
	In the setting of Definition~\ref{def:RDP}, an algorithm satisfies
	Conditional Proportional Representation (CPR) if, almost surely over
	$y$:
	\begin{align*} \forall i \in \left[ k \right], \; 
		\Pr\left( \wh{x} \in c_i | y \right) = \Pr\left( x^* \in c_i | y
		\right).
	\end{align*}
\end{definition}

Intuitively, many images could yield the same lossy measurement.
Because we have no way of knowing exactly from which image the
measurement came, we reconstruct one at random based on how likely
images in the same group are to have yielded this measurement in the
first place. As such, it is ``fair": every image that could have led
to the measurement gets a chance at being represented, not just the
most likely. This also implies that the reconstruction cannot be
deterministic.  Unfortunately, while CPR can be achieved via Posterior
Sampling (Theorem~\ref{thm:cr_cpr}), the fact that the definition
involves the posterior distribution makes it difficult to verify
without full knowledge of the measurement process and the probability
distribution.

It turns out that one cannot achieve RDP and PR simultaneously if you
have a majority which has mass larger than $1/2$. 
\define{prop:rdp_pr_incompat}{Proposition}{
Whenever there exists a majority class that the measurements cannot
100\% distinguish from the non-majority classes, PR and RDP are not
simultaneously achievable.
}
\state{prop:rdp_pr_incompat}

\section{Posterior Sampling }
The \emph{Posterior Sampling} algorithm outputs a reconstruction
$\wh{x}$ drawn from the posterior $P(\cdot \mid y)$. It is known to be
instance-optimal for compressed sensing~\cite{jalal2021instance} and
to give fairly accurate results in practice when implemented via
annealed Langevin dynamics~\cite{stefano-langevin}.  In this section,
we show that it also has good fairness properties.

It is easy to see that if one has access to the distribution $P$ over
images and the likelihood function associated with the measurement
process, then Posterior Sampling will satisfy the CPR.  The following
Theorem shows that this is the \emph{only} algorithm that can satisfy
CPR.

\begin{restatable}{theorem}{posteriorsamplingonly}\label{thm:cr_cpr}
	Posterior Sampling is the only algorithm that achieves oblivious
	Conditional Proportional Representation.
\end{restatable}

\begin{definition}[Symmetric Pairwise Error] In the setting of
	Definition~\ref{def:RDP}, an algorithm satisfies Symmetric Pairwise
	Error (SPE) if
	\begin{align*}
		\Pr\left( \wh{x} \in c_i , x^* \in c_j \right) = \Pr\left( \wh{x}
		\in c_j, x^* \in c_i \right), \; \forall i,j \in \left[ k \right].
	\end{align*}
\end{definition}

Using the fact that the ground truth and reconstruction are
conditionally independent given the measurements, we can show that any
algorithm that satisfies CPR will also satisfy SPE.

\define{thm: cpr implies spe}{Theorem}
{
	In the setting of Definition~\ref{def:RDP}, Conditional Proportional
	Representation implies Symmetric Pairwise Error.
}
\state{thm: cpr implies spe}

Theorem~\ref{thm:cr_cpr} and Theorem~\ref{thm: cpr implies spe} give
the following Corollary.

\define{cor:cond_resamp_achieves_pairwise_error}{Corollary}
{
	Posterior Sampling achieves symmetric pairwise error for any pair of
	sets $U, V \subset \mathbb{R}^n$. 
}
\state{cor:cond_resamp_achieves_pairwise_error}

Finally, for any partition $C$, there exists a reweighting of the
underlying distribution such that Posterior Sampling achieves RDP with
respect to the partition $C$. 

\begin{restatable}{theorem}{thmcondreweight}\label{thm:cond_resamp_reweight}
Let $C = \{c_1, \dots, c_k\}$ be a partition.  There exists a choice
of weights $\lambda_i > 0$ with $\sum \lambda_i = 1$ such that
Posterior Sampling with respect to the reweighted distribution
\[
p_{\lambda}(x) = \sum_i \lambda_i p(x \mid x \in c_i)
\]
satisfies RDP with respect to $C$.
\end{restatable}

In the special case of 2 classes, the reweighting is very simple:
$\lambda_1 = \lambda_2 = \frac{1}{2}$.

\subsection{Representation Cross-Entropy} For the special case when
the collection $C$ is a partition, we can show that Posterior Sampling
obliviously minimizes a loss we call Representation Cross-Entropy
(RCE). Intuitively, one can think of this as the generative analogue
of the cross-entropy loss popular in classification settings.
Following the notation in Definition~\ref{def:RDP}, we define RCE as:
\begin{definition}[Representation Cross-Entropy]
	Let $C=\left\{ c_1, \cdots, c_k \right\}$ form a disjoint partition of
	$\R^n$, and let $U$ be a function such that $U(x)$ encodes where $x$
	lies in the partition.  The Representation Cross-Entropy (RCE) of
	a reconstruction algorithm $\cA$ with respect to $C$ is defined as
	\begin{align*}
		RCE(\cA) := -\E_{x^*, y}\log \Pr_{\wh{x}|y}\left[ \wh{x}
		\in U(x^*) \right].
	\end{align*}
	\label{def:rce}
\end{definition}

We show that if we want to minimize RCE over a partition, then we must
have CPR on this partition:
\define{thm:opt-rce}{Theorem}
{	Let $C=\left\{ c_1, \cdots, c_k  \right\}$ form a disjoint partition
of $\R^n$. An algorithm minimizes Representation Cross-Entropy on $C$
iff the algorithm satisfies CPR on $C$.  
}
\state{thm:opt-rce}
  
From Theorem~\ref{thm:cr_cpr}, we know that Posterior Sampling is the
only algorithm that can achieve CPR over all measurable sets. The same
result holds if we restrict to measurable partitions, so Posterior
Sampling is the only algorithm that minimizes RCE obliviously to the
partition.

\section{Experiments}
So far we have discussed and analyzed properties of several different
fairness metrics. In this section, we briefly describe how one can
implement Posterior Sampling,
and study the empirical performance
of Posterior Sampling and PULSE with respect to our defined metrics, on
the MNIST~\cite{lecun1998mnist}, FlickrFaces-HQ~\cite{karras2019style} and AFHQ cat\&dog
dataset~\cite{choi2020starganv2}. 

\subsection{Langevin Dynamics}\label{sec:langevin}
We implement Posterior Sampling via Langevin dynamics, which states that if $x_0 \sim \cN(0,cI_n),$ (for $c$ appropriately small),
then we can sample from $p(x|y)$ by running noisy gradient ascent:
\begin{align*}
    x_{t+1} \leftarrow x_t + \gamma_t \nabla_{x_t} \log p(x_t | y) + \sqrt{2\gamma_t}\xi_t,
\end{align*}
where $\xi_t \sim \cN(0,I_n)$ is an i.i.d. standard Gaussian drawn at each iteration.
It is well known~\cite{welling2011bayesian, stefano-langevin} that as $\gamma_t \to 0, t\to \infty,$ we have $p(x_t|y) \to p(x|y).$
In practice, we need some form of annealing of the noise in order to mix efficiently. 
Since our algorithm is randomized, we always output the first obtained reconstruction.
Please see Appendix~\ref{app:langevin} for architecture-specific details.

\subsection{MNIST dataset}
We trained a VAE~\cite{kingma2013auto} on MNIST digits, and consider
the groups \{0, 1, 2, 3, 4, $\geq 5$\}. As seen in the
confusion matrix in Table~\ref{tab:mnist}, Posterior Sampling does
satisfy SPE obliviously over the groups (recall that SPE also implies PR).
\begin{table}[t]
    \centering
    \resizebox{0.7\linewidth}{!}{
    \begin{tabular}{|c|c|c|c|c|c|c|}
    \hline
      & 0 & 1 & 2 & 3 & 4 & $\geq 5$ \\
      \hline
    0 & 33 & 2 & 0 & 2 & 0 & 6 \\
    1 & 0 & 61 & 1 & 0 & 1 & 4 \\
    2 & 0 & 2 & 45 & 1 & 1 & 6 \\
    3 & 1 & 0 & 2 & 33 & 2 & 7 \\
    4 & 1 & 0 & 1 & 0 & 41 & 12 \\
    $\geq 5$ & 2 & 5 & 10 & 5 & 14 & 199 \\
    \hline
    \end{tabular}
    }
    \caption{\small Confusion matrix for super-resolution of MNIST digits
      after downsampling by $4\times$ in each dimension. The rows
      denote the labels of original images, the columns denote the
      labels of reconstructed images. The symmetric nature of the
      matrix shows that Posterior Sampling achieves SPE
    obliviously over multiple groups.}
    \label{tab:mnist}
\end{table}

\subsection{FlickrFaces dataset}
\paragraph{Dataset and generative models.}
We use a StyleGAN2~\cite{karras2020analyzing} model for PULSE, while
Posterior Sampling uses the NCSNv2 generative
model~\cite{song2019generative,song2020improved}. We choose this model
as it has been designed to produce images via Langevin dynamics, which
is the practical implementation of Posterior Sampling. 

\paragraph{Results}
In Figure~\ref{fig:reconstr-ffhq}, we show the results of
super-resolution on Barack Obama and four faces from FFHQ, using PULSE
and Posterior Sampling. As shown, Posterior Sampling preserves
the image features better than PULSE. 
We use the CLIP classifier~\cite{radford2021learning} to assign labels
of \{child with / without glasses, adult with / without glasses\}, and
report the confusion matrix in Table~\ref{tab:ffhq}. This shows that
Posterior Sampling satisfies SPE over multiple groups obliviously.

Please see Appendix~\ref{app:ffhq} for more representative samples.
These correspond to images 69000-69020 in the FFHQ validation set (these were the first 20 images as we downloaded them in reverse-chronological order from the Google Drive folder).
\begin{table}
    \centering
    \begin{tabular}{|c|c|c|c|c|}
                \hline 
                 & A & B & C & D  \\
                 \hline
                 A & 5 & 3 & 2 & 0 \\
                 B & 1 & 99 & 0 & 10 \\
                 C & 1 & 0 & 68 & 10 \\
                 D & 1 & 10 & 8 & 282 \\
                 \hline
    \end{tabular}
    \caption{\small Confusion matrix for super-resolution of FFHQ
      faces after $32\times$ downsampling in each dimension. The
      categories are A: child with glasses, B: child without glasses,
      C: adult with glasses, D: adult without glasses. Rows denote
      labels of original images, columns denote labels of
      reconstructed images.  The symmetric nature of the matrix shows
      that Posterior Sampling achieves SPE over multiple groups
      obliviously.}
    \label{tab:ffhq}
\end{table}

\subsection{AFHQ Cats and Dogs dataset}
\paragraph{Dataset and models}
We trained StyleGAN2~\cite{karras2020training} on the AFHQ cat \&
dog~\cite{choi2020starganv2} training set. In order to study the
effect of population bias on PULSE and Posterior Sampling, we trained
three models on datasets with varying bias: (1) 20\% cats and 80\%
dogs, (2) 80\% cats and 20\% dogs, and (3) 50\% cats and 50\% dogs. 

In order to label the images generated by the GAN, we take a
pre-trained Resnet108 and retrain the last layer using labelled images
from the AFHQ training set. We find that the classifier's predictions
does match the human perception of dogs and cats in general.

\paragraph{Posterior Sampling satisfies SPE and PR when the cats and dogs are unbalanced.}
For this experiment, we draw $x^*$ from the AFHQ validation dataset,
which contains 500 images of cats and 500 images of dogs.  Since we
want to study whether Posterior Sampling and PULSE satisfy SPE and
PR, we construct the test set to match the training population of the
generator. That is, for the 20\% cat generator, we use 125 images of
cats and all 500 images of dogs from the AFHQ dataset. Similarly, for
the 80\% cat generator, we use 500 images of cats and 125 images of
dogs in the test set.

We then downscale the images, and vary the downscaling factor such
that the observed measurements have resolution $1\times 1, 2\times 2,
4\times 4, 8\times 8.$ We ran PULSE and Posterior Sampling to
super-resolve the blurry measurements, and used a classifier to count
how many cats and dogs were reconstructed in the wrong class. The
results for the 20\% cat generator are in Table~\ref{tab:cat20dog80} and Figure~\ref{fig:cat20dog80},
and the results for the 80\% cat generator are in
Table~\ref{tab:cat80dog20} and Figure~\ref{fig:cat80dog20}. In
Figure~\ref{fig:cat20dog80-examples} we show example reconstructions.

We find that PULSE consistently makes very few mistakes on the
majority, and an overwhelming number of mistakes on the minority.
Posterior Sampling, however, makes an approximately equal number
of mistakes on each class (i.e., satisfies SPE).  Equivalently for
this 2-class setting, it generates cats and dogs in proportion to
their population (i.e., satisfies PR).

\begin{figure*}[h]
	\begin{subtable}[t]{0.48\linewidth}
		\vspace{-13em}
		\centering
		\resizebox{0.8\columnwidth}{!}{%
		\begin{tabular}[scale=0.9]{|c|c|c|c|c|}
			\hline 
			\multirow{2}{*}{$m$} & \multicolumn{2}{|c|}{PULSE} &
			\multicolumn{2}{|c|}{Ours} \\
			\cline{2-5}
			& Cats & Dogs & Cats &
			Dogs \\
			\hline
			$1\times 1$  & 113 & 58 & 102 & 106 \\
			$2\times 2$  & 104 & 44 & 97 & 81 \\
			$4\times 4$  & 110 & 25 & 94 & 110 \\
			$8\times 8$ & 101 & 10 & 70 & 59 \\
			\hline
		\end{tabular}
		}
		\caption{\small Number of errors on 20\% cat generator, for each
		resolution. Sampled test set has \textbf{125} cats and
		\textbf{500} dogs from the AFHQ validation set to mimick the
		generator's training distribution.  PULSE makes errors on almost
		all the cats and a few dogs, while Posterior Sampling is
		relatively balanced.}
		\label{tab:cat20dog80}
	\end{subtable}
	\hfill
	\begin{subfigure}[t]{0.48\linewidth}
		\centering
		\includegraphics[scale=0.35]{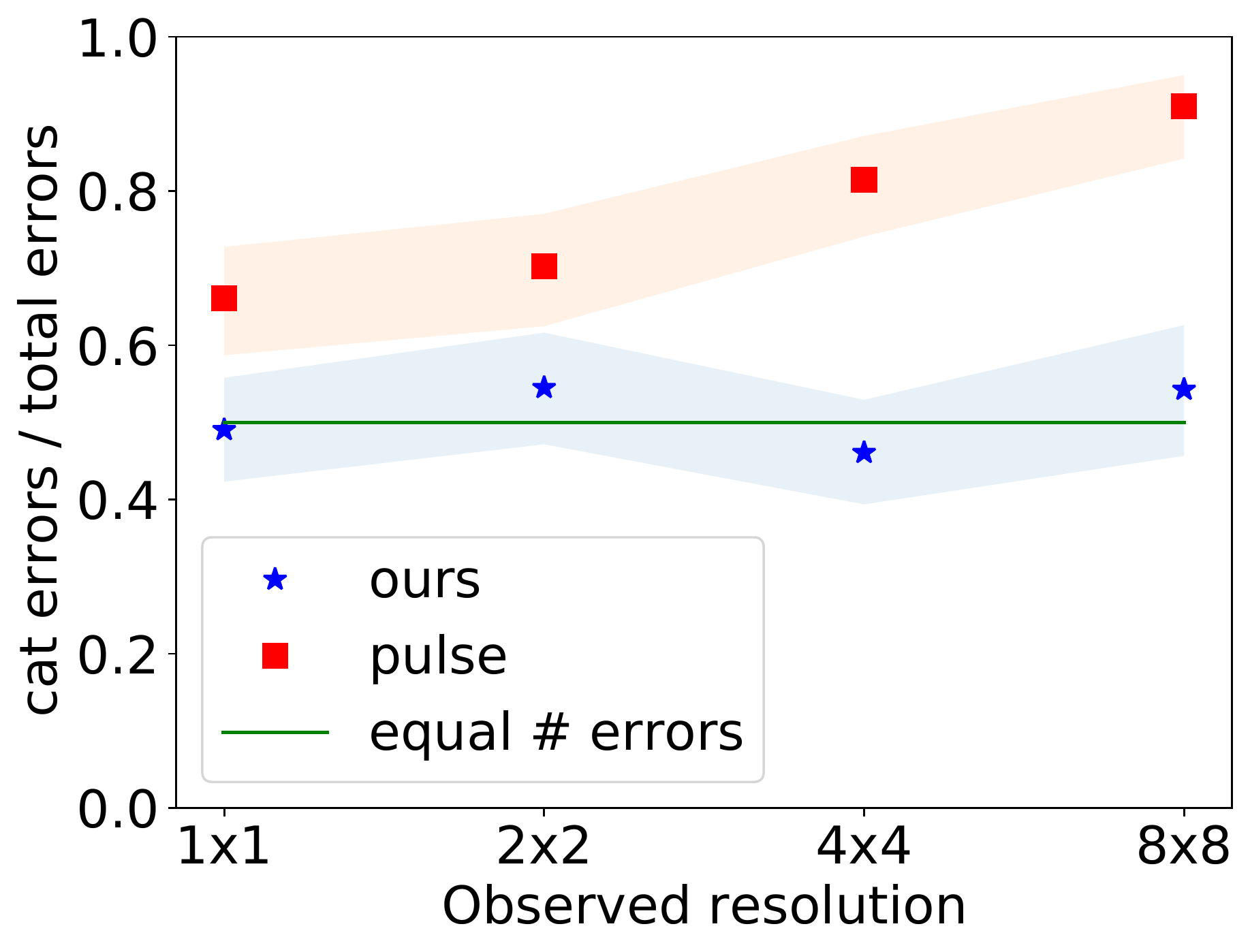}
		\caption{\small Fraction of all errors on cats for 20\% cat generator.}
    	\label{fig:cat20dog80}
	\end{subfigure}
	\bigskip
	\begin{subtable}[t]{0.48\linewidth}
		\vspace{-13em}
		\centering
		\resizebox{0.8\columnwidth}{!}{%
		\begin{tabular}[scale=0.9]{|c|c|c|c|c|}
			\hline 
			\multirow{2}{*}{$m$} & \multicolumn{2}{|c|}{PULSE} &
			\multicolumn{2}{|c|}{Ours} \\
			\cline{2-5}
			& Cats & Dogs & Cats &
			Dogs \\
			\hline
			$1\times 1$  & 0 & 125 & 115 & 96 \\
			$2\times 2$  & 0 & 125 & 82 & 97 \\
			$4\times 4$  & 0 & 125 & 94 & 86 \\
			$8\times 8$  & 1 & 123 & 71 & 94 \\
			\hline
		\end{tabular}
		}
		\caption{\small Number of errors on 80\% cat generator, for each
		resolution. Sampled test set has \textbf{500} cats and
		\textbf{125} dogs from the AFHQ validation set to mimick the
		generator's training distribution.  PULSE makes errors on almost
		all the dogs and no cats, while Posterior Sampling is
		relatively balanced.}
		\label{tab:cat80dog20}
	\end{subtable}
	\hfill
	\begin{subfigure}[t]{0.48\linewidth}
		\centering
		\includegraphics[scale=0.35]{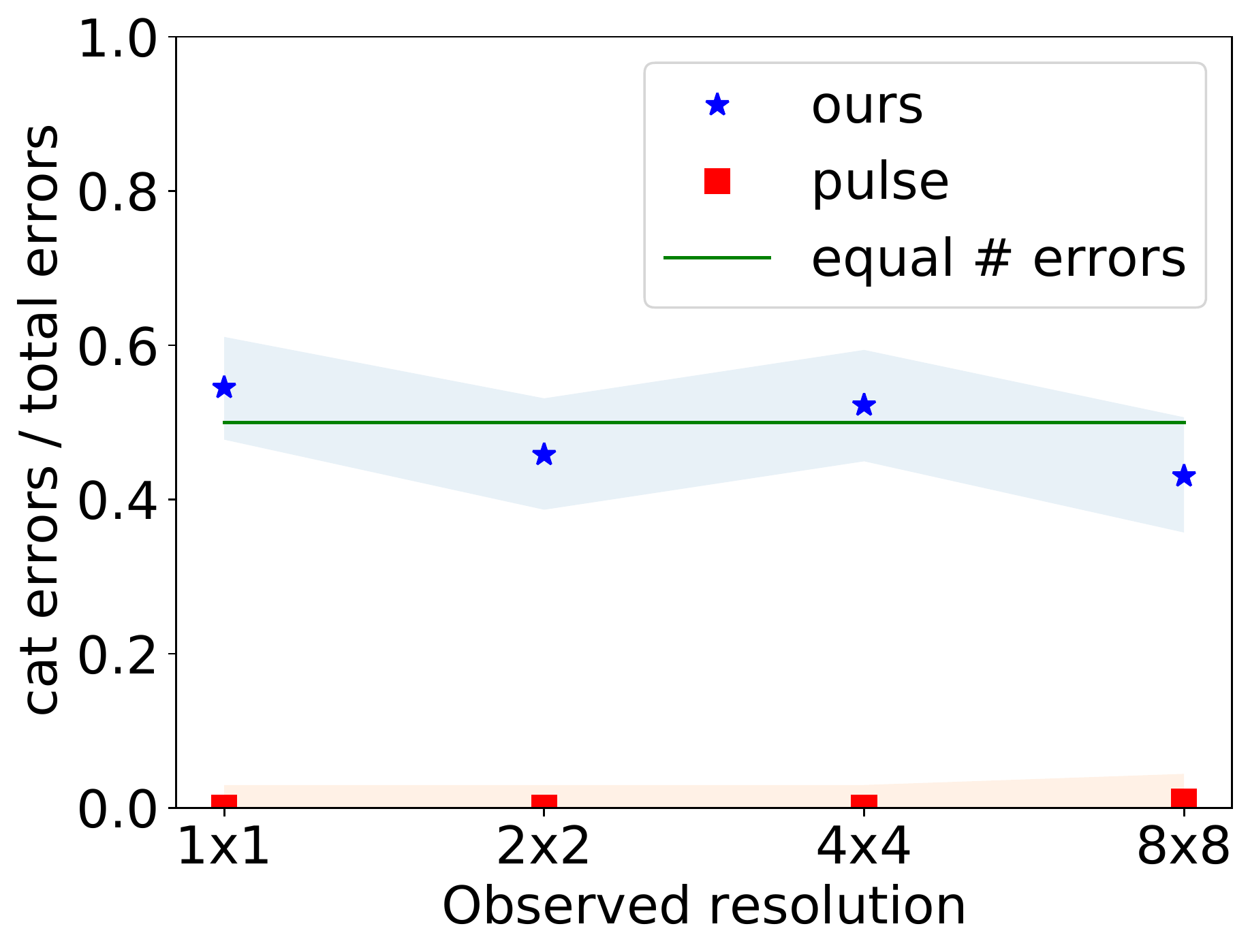}
		\caption{\small Fraction of all errors on cats for 80\% cat generator.}
    	\label{fig:cat80dog20}
	\end{subfigure}
	\caption{\small Figure (a): we use a StyleGAN2 model trained on 20\% cats and
	report errors when reconstructing images from
	low-resolution measurements. The test set consists of 125 cats and 500 dogs from
	the AFHQ validation set to mimick the generator's training
	distribution (note that these correspond to all dogs in the AFHQ
	validation set). Figure (b) shows the proportion of all errors that
	are on cats, along with 95\% confidence intervals from a binomial
	test. An algorithm that satisfies SPE would have this
	probability=0.5 (green line).  Figure (c), (d), show analogous results 
	when we use a StyleGAN2 generator trained on 80\% cats. PULSE is clearly 
	biased towards the majority, while Posterior Sampling via Langevin dynamics 
	appears to satisfy SPE and PR.
  We remark that at $1\times 1$ resolution, there is effectively no
  information and Posterior Sampling random guesses, while PULSE
  prefers the majority.}
	\label{fig:afhq-quant}
\end{figure*}

\paragraph{Posterior Sampling satisfies RDP, SPE, and PR when the cats
and dogs are balanced.} We use a generator trained on 50\% cats and
50\% dogs, and study whether Posterior Sampling and PULSE satisfy RDP,
SPE, and PR in practice. In this case, we use all images of cats and
dogs from the AFHQ validation set. These results are in
Appendix~\ref{app:afhq}, Figure~\ref{fig:cat50dog50}.  Please see
Appendix~\ref{app:afhq} for more results as we vary the training bias
of the generator and test SPE for images drawn from the range of the
generator.

\section{Limitations}
The fact that CPR can be satisfied obliviously is its main strength,
as the subgroups one would like to protect are often not well defined
or labeled in datasets. This is especially beneficial for overlooked
groups that lack the power to convince an algorithm designer to cater
to them.  However, obliviousness can also be seen as a weakness, as it
leads to symmetry in the \emph{number} of errors in each group rather
than the \emph{fraction} of errors.  For two groups, this means that
the minority group will always have higher error rate than the
majority.

Furthermore, the goal of CPR is to treat the members of each group
equally.  The philosophical stance behind this property implicitly
views being ``fair'' as treating individuals equally, and hence
representing groups in proportion to their size.  However, alternative
philosophical stances exist. In particular, it is at odds with the
idea that historically oppressed minorities should get particular
attention \cite{hanna2020towards}.  One could adapt such an approach
into our framework by reweighting the classes, analogous to
Theorem~\ref{thm:cond_resamp_reweight}, but doing so requires explicit
group information.

Finally, all of the definitions we consider focus on representation
but do not consider the quality of the reconstruction. If all
reconstructions on minorities were of poor quality (for instance
because the training set did not have enough images of this specific
minority, and/or they were of poorer quality, as we know can happen
\cite{buolamwini2018}), the algorithm could still satisfy any of the
definitions and be deemed ``fair'' according to it.  Representation
fairness is just one piece of the larger question of fairness in
reconstruction.

\section{Conclusion}
In the image generation setting, fairness is related to the concept of
\emph{representation}: we assign a protected group to the output,
which should match the protected group of the input. This is a stark
contrast with the classification setting, in which we usually require
some form of independence between the output and the sensitive
attributes. We therefore introduce two notions of fairness, an
extension of demographic parity called Representation Demographic
Parity (RDP), and a conceptually new notion, Conditional Proportional
Representation (CPR). We show that these notions are in general
incompatible. Furthermore, we prove that RDP is strongly dependent on
the choice of the protected groups. This is especially problematic for
generating images of people, as races are usually ill-defined and/or
ambiguous. CPR, however, does not suffer from these downsides, and can
even be satisfied obliviously (i.e., simultaneously for any choice of
protected groups).

We prove that Posterior Sampling can achieve CPR, and is actually the
only algorithm that can achieve CPR fully obliviously. We show how to
experimentally implement our findings through Langevin dynamics, and
our experiments exhibit the expected desirable properties.

We see our work as a first step towards better understanding ideas of
fairness in the context of generating structured data -- our paper
deals with image generation, but the problem of generating structured
data could be extended to other settings. What happens when the data
is of a different type? For instance, one might want to predict
pronouns in the context of text completion or generation, or provide
the option to use a certain dialect.

The definitions introduced in this paper are specific to generative
procedures. However, the underlying issues -- having definitions that
do not strongly rely on the choice of the protected groups -- can be
found in the classification setting as well. It would be interesting
to see if any analogs of CPR exist in this more traditional setting,
and if there exist algorithms that can achieve it obliviously.

\section{Acknowledgements} Ajil Jalal and Alex Dimakis were
supported by NSF Grants CCF 1934932, AF 1901292, 2008710, 2019844 the
NSF IFML 2019844 award and research gifts by Western Digital,
WNCG and MLL, computing resources from TACC and the Archie Straiton
Fellowship.  Sushrut Karmalkar was supported by a University Graduate
Fellowship from the University of Texas at Austin. Jessica Hoffmann
was supported by NSF TRIPODS grant 1934932. Eric Price was supported by
NSF Award CCF-1751040 (CAREER) and NSF IFML 2019844.

\bibliography{main}
\bibliographystyle{icml2021}

\clearpage
\appendix
\onecolumn
\section{FFHQ Experiments}\label{app:ffhq}
\begin{figure}[h]
\begin{center}
  \includegraphics[width=\linewidth]{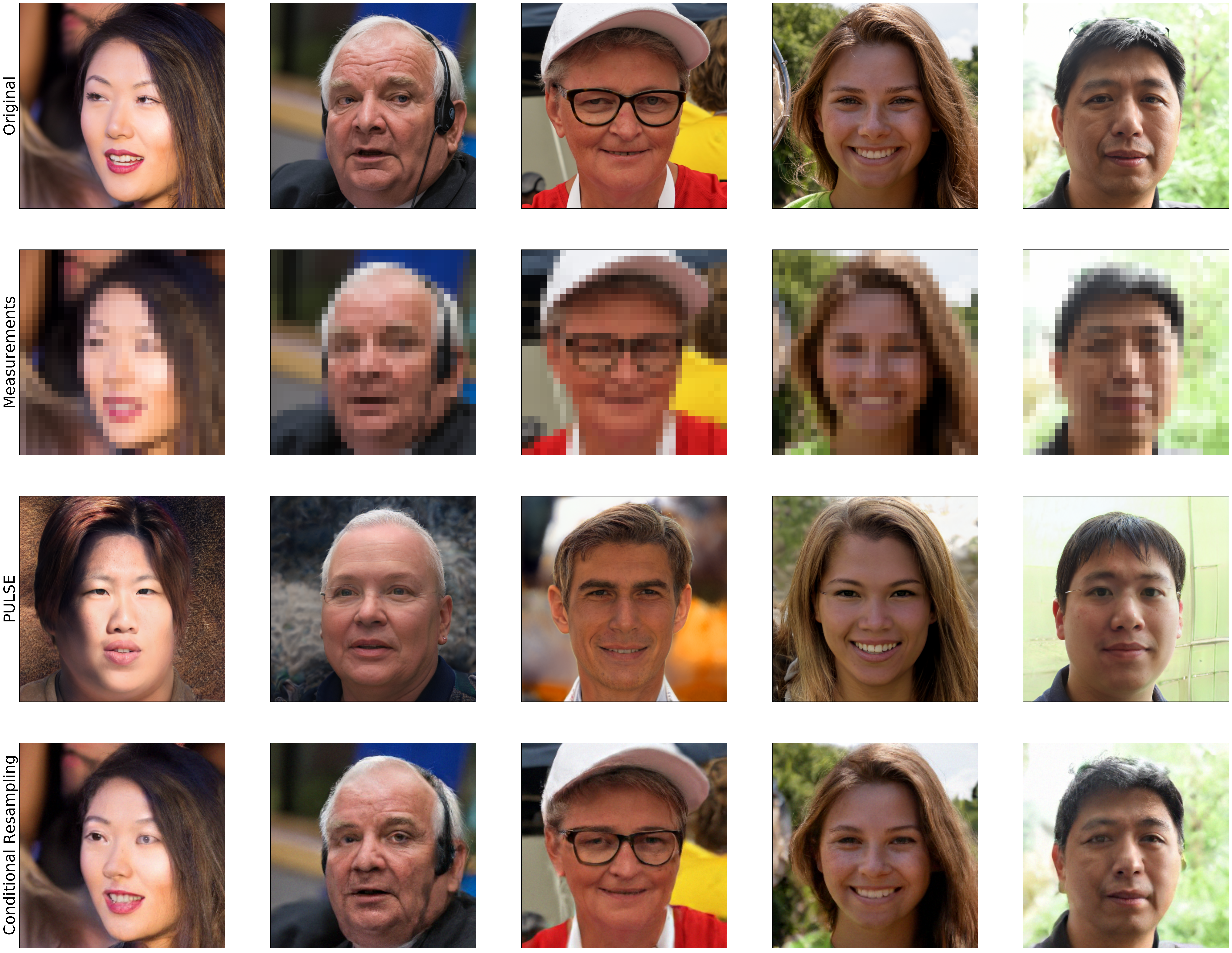}
\end{center}
\caption{Super-resolution reconstructions on faces 69000-69004 from
the FFHQ dataset. The top row shows original images, the second row
shows what the algorithms observe: blurry measurements after
downsampling by $32\times$ in each dimension. The third row shows
reconstructions by PULSE, and the last row shows reconstructions by
Posterior Sampling via Langevin dynamics, the algorithm we are
advocating for.
		}
\label{fig:ffhq-appendix-1}
\end{figure}

\begin{figure}[h]
\begin{center}
  \includegraphics[width=\linewidth]{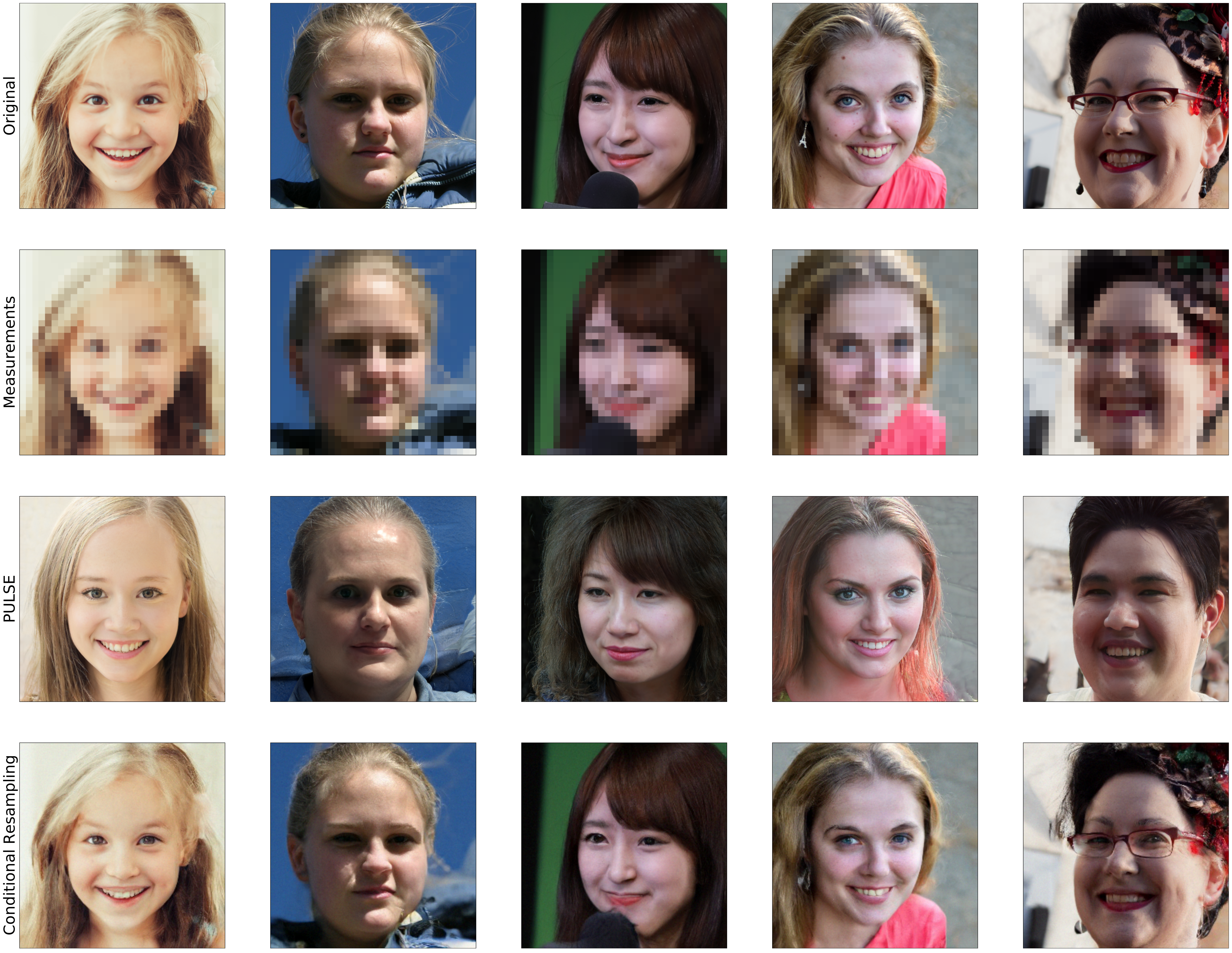}
\end{center}
\caption{Super-resolution reconstructions on faces 69005-69009 from
the FFHQ dataset. The top row shows original images, the second row
shows what the algorithms observe: blurry measurements after
downsampling by $32\times$ in each dimension. The third row shows
reconstructions by PULSE, and the last row shows reconstructions by
Posterior Sampling via Langevin dynamics, the algorithm we are
advocating for.
		}
\label{fig:ffhq-appendix-2}
\end{figure}

\begin{figure}
\begin{center}
  \includegraphics[width=\linewidth]{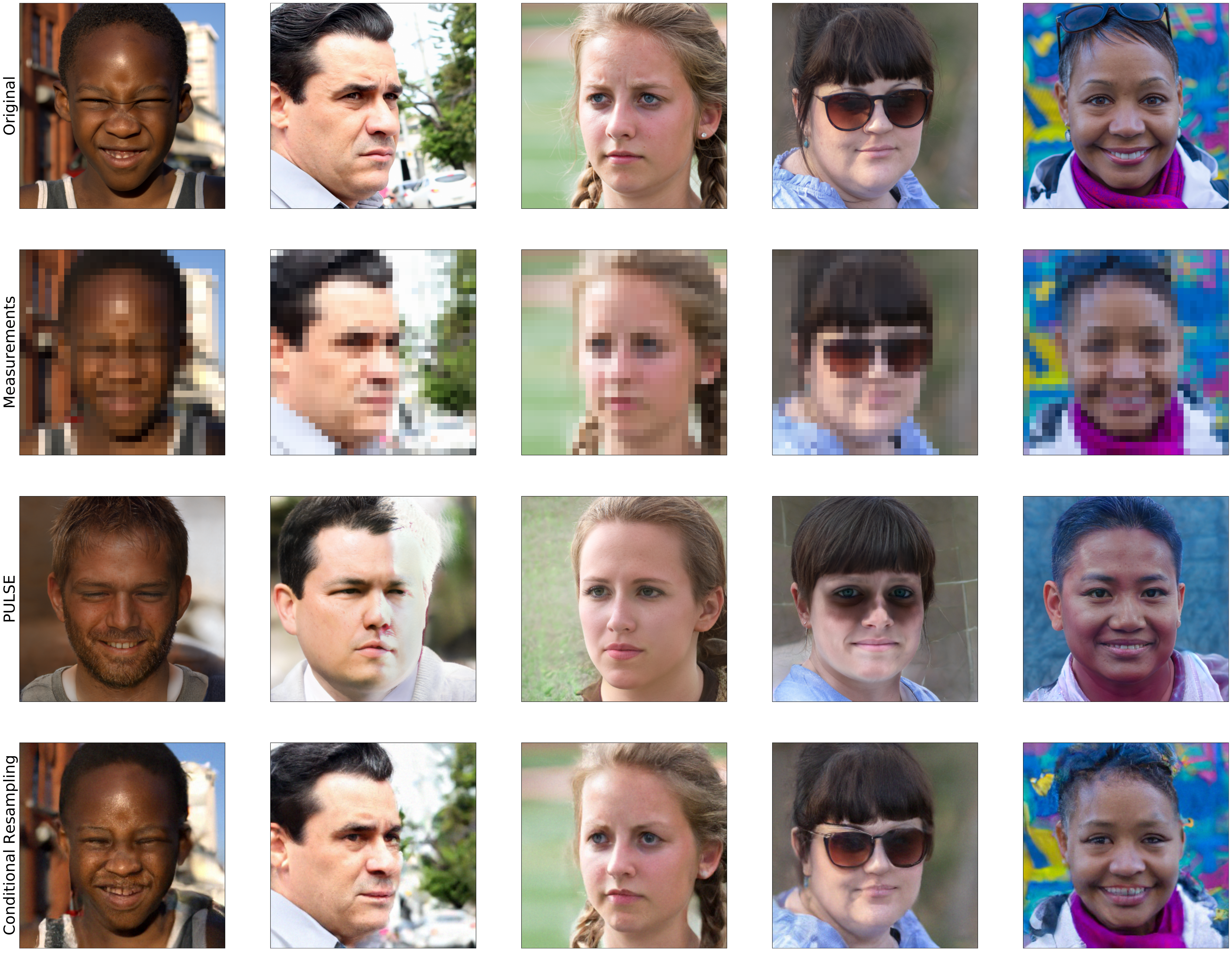}
\end{center}
\caption{Super-resolution reconstructions on faces 69010-69014 from
the FFHQ dataset. The top row shows original images, the second row
shows what the algorithms observe: blurry measurements after
downsampling by $32\times$ in each dimension. The third row shows
reconstructions by PULSE, and the last row shows reconstructions by
Posterior Sampling via Langevin dynamics, the algorithm we are
advocating for.
		}
\label{fig:ffhq-appendix-3}
\end{figure}

\begin{figure}
\begin{center}
  \includegraphics[width=\linewidth]{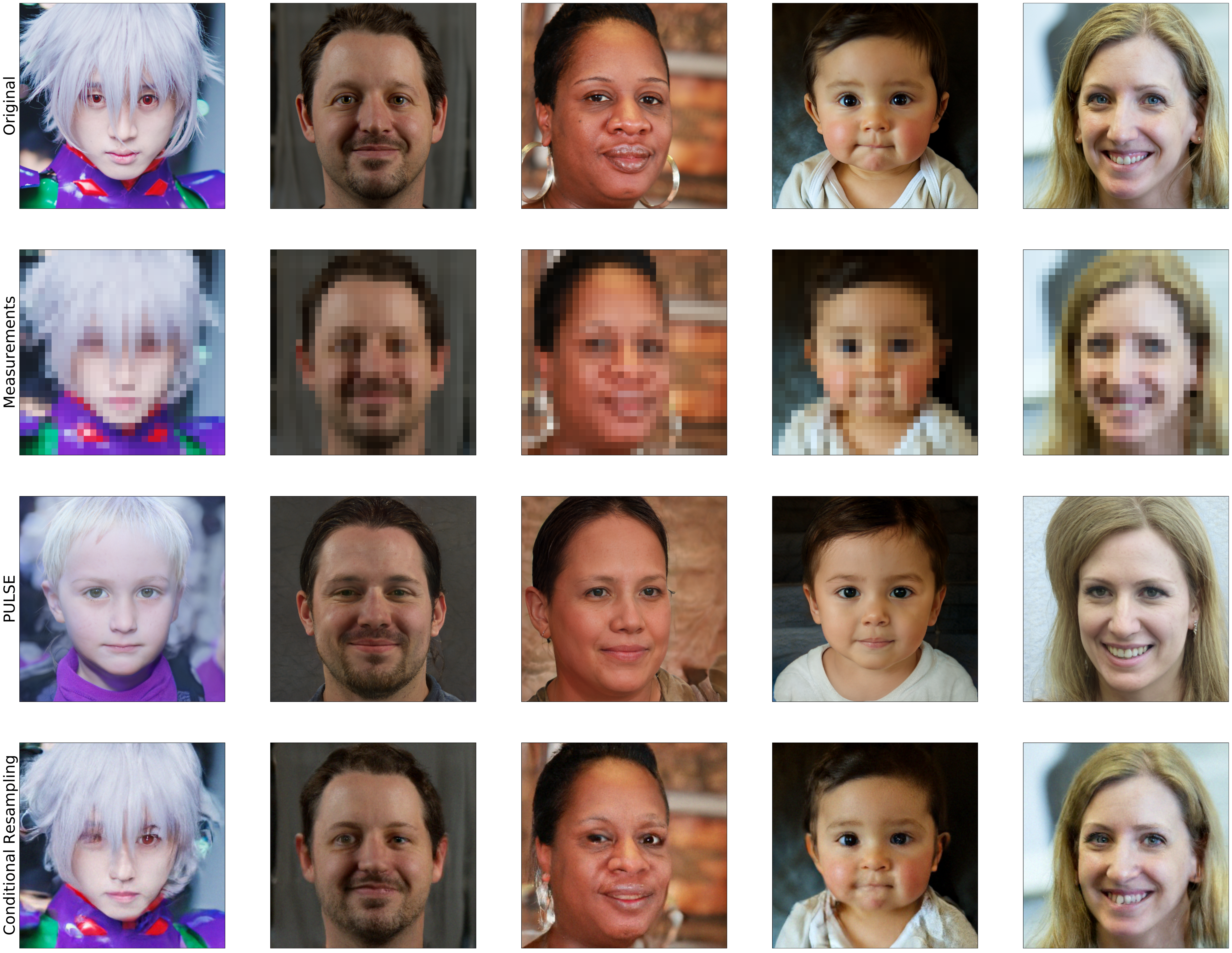}
\end{center}
\caption{Super-resolution reconstructions on faces 69015-69020 from
the FFHQ dataset. The top row shows original images, the second row
shows what the algorithms observe: blurry measurements after
downsampling by $32\times$ in each dimension. The third row shows
reconstructions by PULSE, and the last row shows reconstructions by
Posterior Sampling via Langevin dynamics, the algorithm we are
advocating for.
		}
\label{fig:ffhq-appendix-4}
\end{figure}

\clearpage
\section{AFHQ Experiments}\label{app:afhq}
\subsection{50\% cat generator}
For this experiment, we draw $x^*$ from the validation set of the AFHQ
dataset which contains 500 images of cats + 500 images of dogs.  We
use a generator trained on 50\% cats and 50\% dogs, and use it to
study whether posterior sampling and PULSE satisfy RDP, SPE, and PR in
practice. These results are in Figure~\ref{fig:cat50dog50}.
\begin{figure*}[t]
	\begin{subtable}[t]{0.48\linewidth}
		\vspace{-13em}
		\centering
		\resizebox{\columnwidth}{!}{%
		\begin{tabular}[scale=0.9]{|c|c|c|c|c|}
			\hline 
			\multirow{2}{*}{$m$} & \multicolumn{2}{|c|}{PULSE} &
			\multicolumn{2}{|c|}{Ours} \\
			\cline{2-5}
			& Cats & Dogs & Cats &
			Dogs \\
			\hline
			$1\times 1$ & 319 & 183 & 245 & 261 \\
			$2\times 2$ & 282 & 234 & 239 & 239 \\
			$4\times 4$ & 225 & 246 & 223 & 229 \\
			$8\times 8$ & 160 & 179 & 119 & 146 \\
			\hline
		\end{tabular}
		}
		\caption{Number of errors. Test set has \textbf{500} cats and
		\textbf{500} dogs}
	\end{subtable}
	\begin{subfigure}[t]{0.48\linewidth}
		\centering
		\includegraphics[scale=0.35]{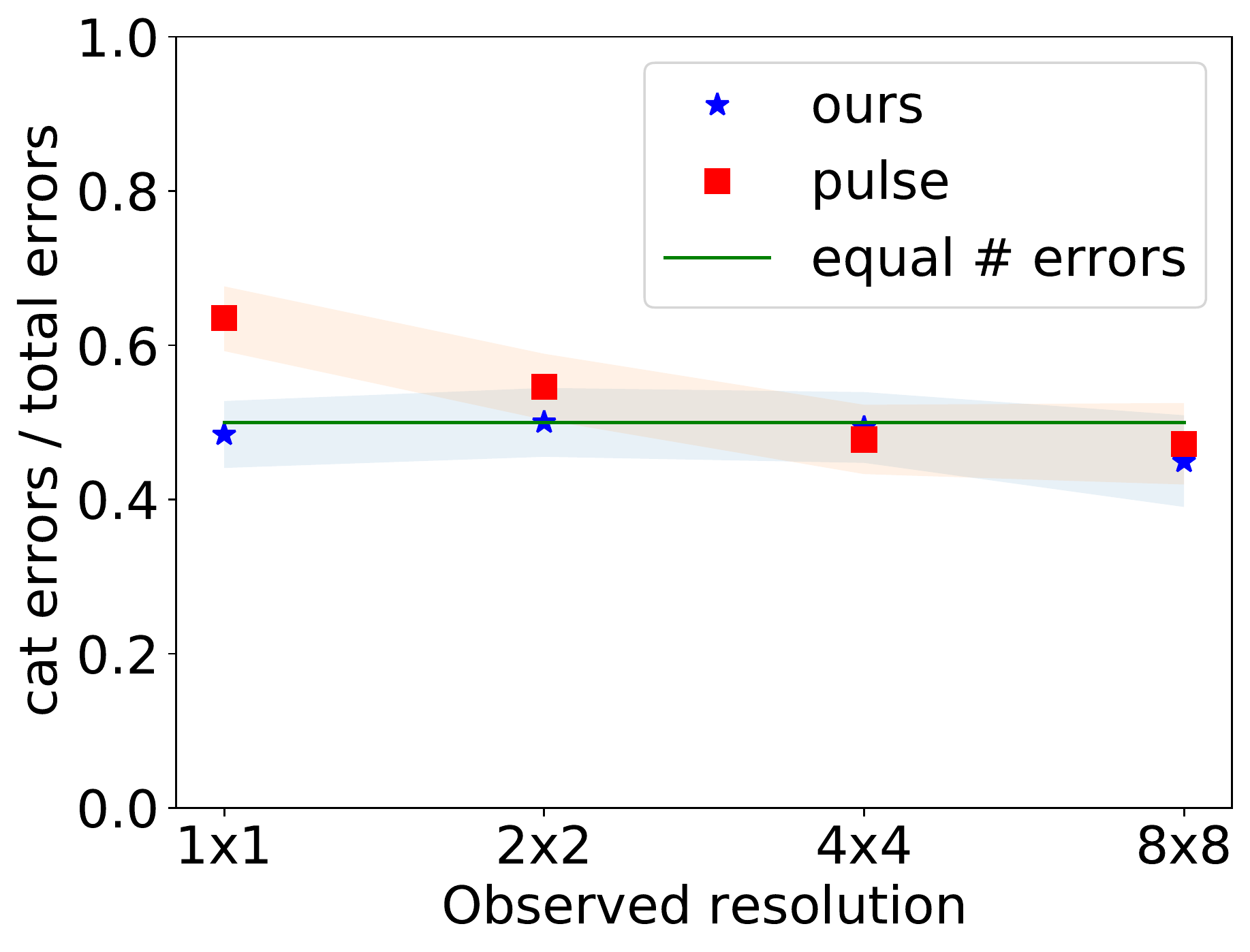}
		\caption{Fraction of all errors on cats for 50\% cat generator.}
	\end{subfigure}
	\caption{We use a StyleGAN2 model trained on 50\% cats and
	report errors when reconstructing images from
	low-resolution measurements. The test set consists of 500 cats and 500 dogs from
	the AFHQ validation set to mimick the generator's training
	distribution (note that these correspond to all cats and dogs in the
	AFHQ validation set).	
	Figure (b) shows the proportion of all errors that are on cats,
	along with 95\% confidence intervals from a binomial test.
	An algorithm
	that satisfies SPE would have this probability=0.5 (green line).
	In this case where the generator is balanced, Posterior Sampling
	via Langevin dynamics appears to achieve SPE, PR, and RDP. PULSE
	also appears to satisfy SPE, PR, and RDP, except when the resolution
	of measurements is $1\times 1.$}
	\label{fig:cat50dog50}
\end{figure*}

\subsection{$x^*$ drawn from generator}
\begin{figure*}[t]
	\begin{subtable}[t]{0.48\linewidth}
		\vspace{-13em}
		\centering
		\resizebox{\columnwidth}{!}{%
		\begin{tabular}[scale=0.9]{|c|c|c|c|c|}
			\hline 
			\multirow{2}{*}{$m$} & \multicolumn{2}{|c|}{PULSE} &
			\multicolumn{2}{|c|}{Ours} \\
			\cline{2-5}
			& Cats & Dogs & Cats &
			Dogs \\
			\hline
					$1\times 1$  & 56 & 17 & 44 & 30   \\
					$2\times 2$  & 48 & 10 & 45 & 24   \\
					$4\times 4$  & 54 & 8 & 33 & 23    \\
					$8\times 8$ & 48 & 4 & 11 & 14    \\
			\hline
		\end{tabular}
		}
		\caption{Number of errors on 20\% cat generator, for each
    resolution. Sampled test set has \textbf{60} cats and \textbf{140}
  dogs.  PULSE makes errors on almost all the cats and relatively few
dogs, while Posterior Sampling is relatively balanced.}
	\end{subtable}
	\begin{subfigure}[t]{0.48\linewidth}
		\centering
		\includegraphics[scale=0.35]{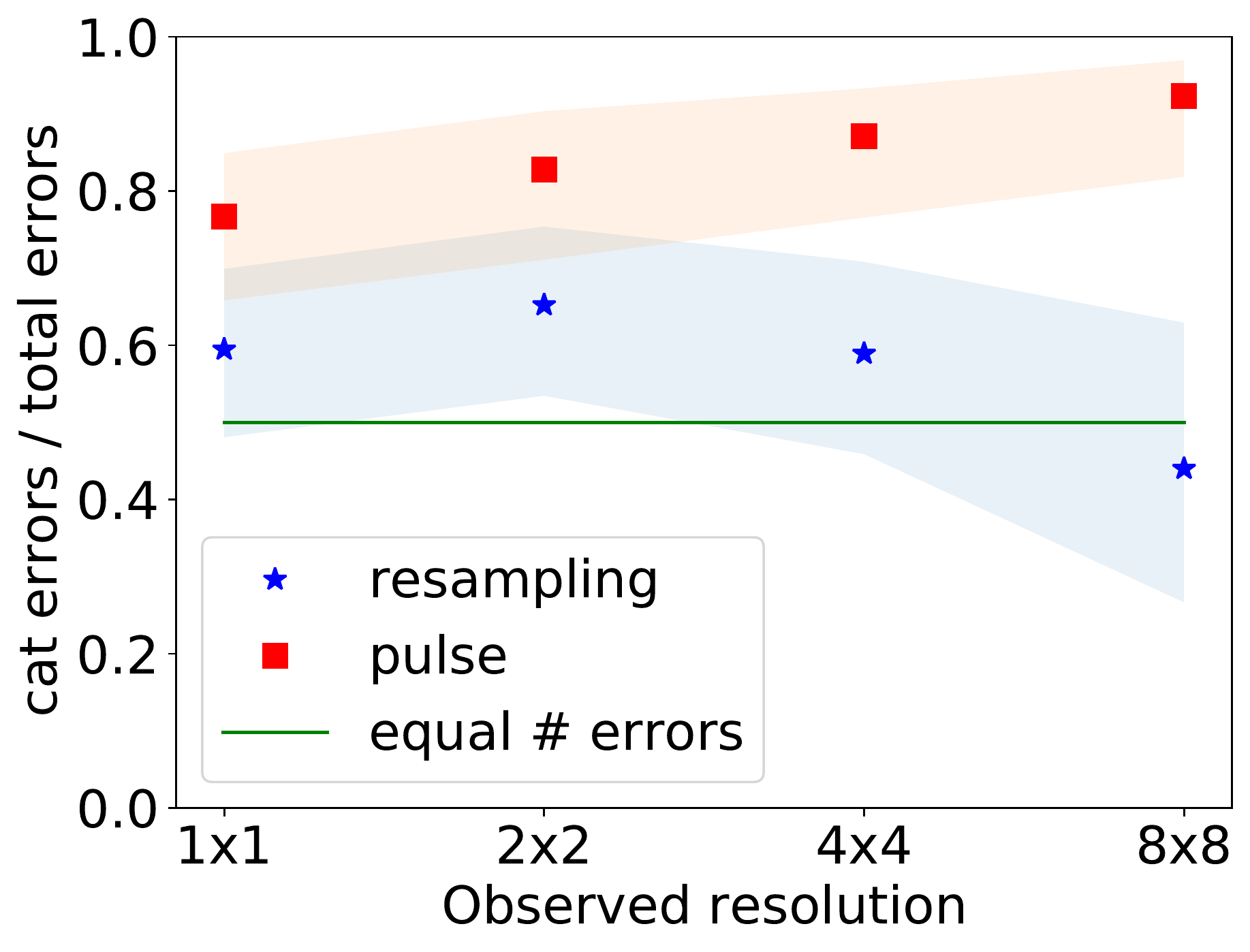}
		\caption{Binomial hypothesis test for Symmetric Pairwise Error
		(SPE)}
	\end{subfigure}
  \caption{We sample 200 images from a StyleGAN2 model trained on 20\%
    cats, and report errors when reconstructing them from
    low-resolution measurements.  Figure (b) shows the proportion of
    all errors that are on cats, along with 95\% confidence intervals
    from a binomial test. An algorithm that satisfies SPE would have
    this probability=0.5 (green line). PULSE is clearly biased towards
    the majority, while Posterior Sampling via Langevin dynamics
    appears to satisfy SPE (except when $m=2\times 2$, but one failure
    is unsurprising as we are performing sequential hypothesis tests.)}
	\label{fig:fakecat20dog80}
\end{figure*}

\begin{figure*}[t]
	\begin{subtable}[t]{0.48\linewidth}
		\vspace{-13em}
		\centering
		\resizebox{\columnwidth}{!}{%
		\begin{tabular}[scale=0.9]{|c|c|c|c|c|}
			\hline 
			\multirow{2}{*}{$m$} & \multicolumn{2}{|c|}{PULSE} &
			\multicolumn{2}{|c|}{Ours} \\
			\cline{2-5}
			& Cats & Dogs & Cats &
			Dogs \\
			\hline
				$1\times 1$& 0 & 47 & 37 & 37 \\
				$2\times 2$& 0 & 47 & 30 & 29  \\
				$4\times 4$& 0 & 47 & 21 & 23 \\
				$8\times 8$& 1 & 47 & 4  & 8  \\
			\hline
		\end{tabular}
    }
    \label{tab:fakecat80dog20}
    \caption{Number of errors on 80\% cat generator, for each
    resolution. Sampled test set has \textbf{153} cats and \textbf{47}
    dogs.  PULSE makes errors on almost all the cats and relatively few
    dogs, while posterior sampling is relatively balanced.}
	\end{subtable}
	\begin{subfigure}[t]{0.48\linewidth}
		\centering
		\includegraphics[scale=0.35]{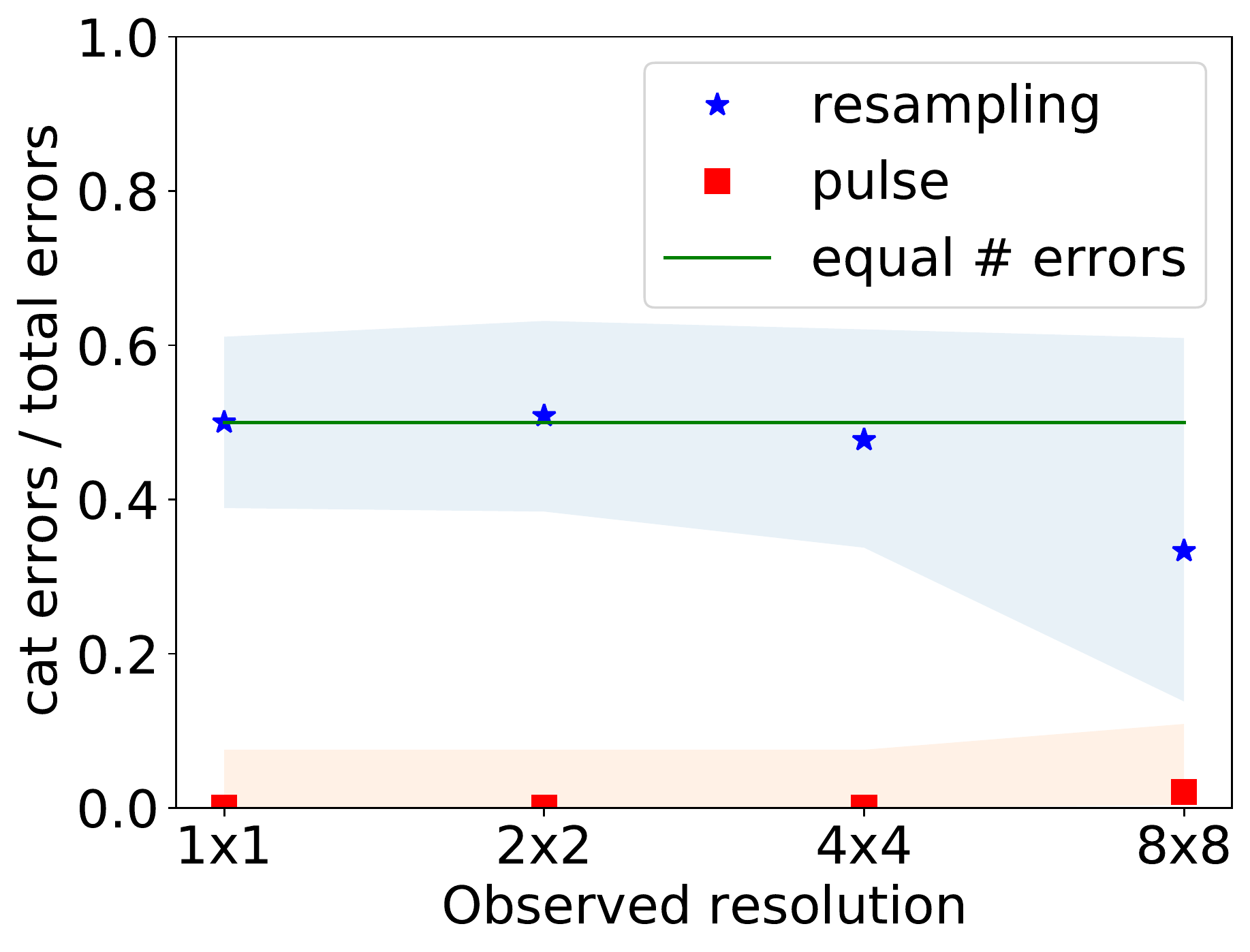}
		\caption{Binomial hypothesis test for Symmetric Pairwise Error
		(SPE)}
	\end{subfigure}
	\caption{We sample 200 images from a StyleGAN2 model trained on 80\%
	cats, and report errors when reconstructing them from low-resolution
	measurements.  Figure (b) shows the proportion of all errors that are on cats, along with 95\% confidence intervals from a binomial test. An algorithm that satisfies SPE would have this probability=0.5 (green line). PULSE is clearly biased towards the majority, while posterior sampling via Langevin dynamics appears to satisfy SPE.}
	\label{fig:fakecat80dog20}
\end{figure*}

In Figure~\ref{fig:fakecat20dog80}, we show results when 200
images drawn from the 20\% cat generator are reconstructed.

In Figure~\ref{fig:fakecat80dog20}, we show results when 200
images drawn from the 80\% cat generator are reconstructed.

\subsection{Varying training bias}
We train StyleGAN2 models with 10\%, 20\%, 30\%, 40\%, 50\%, 60\%,
70\%, 80\%, 90\% cats, and report the fraction of errors on cats when
tested on the AFHQ validation set. The results are in
Figure~\ref{fig:spe-biases}.

\begin{figure}
\begin{center}
  \includegraphics[width=0.8\columnwidth]{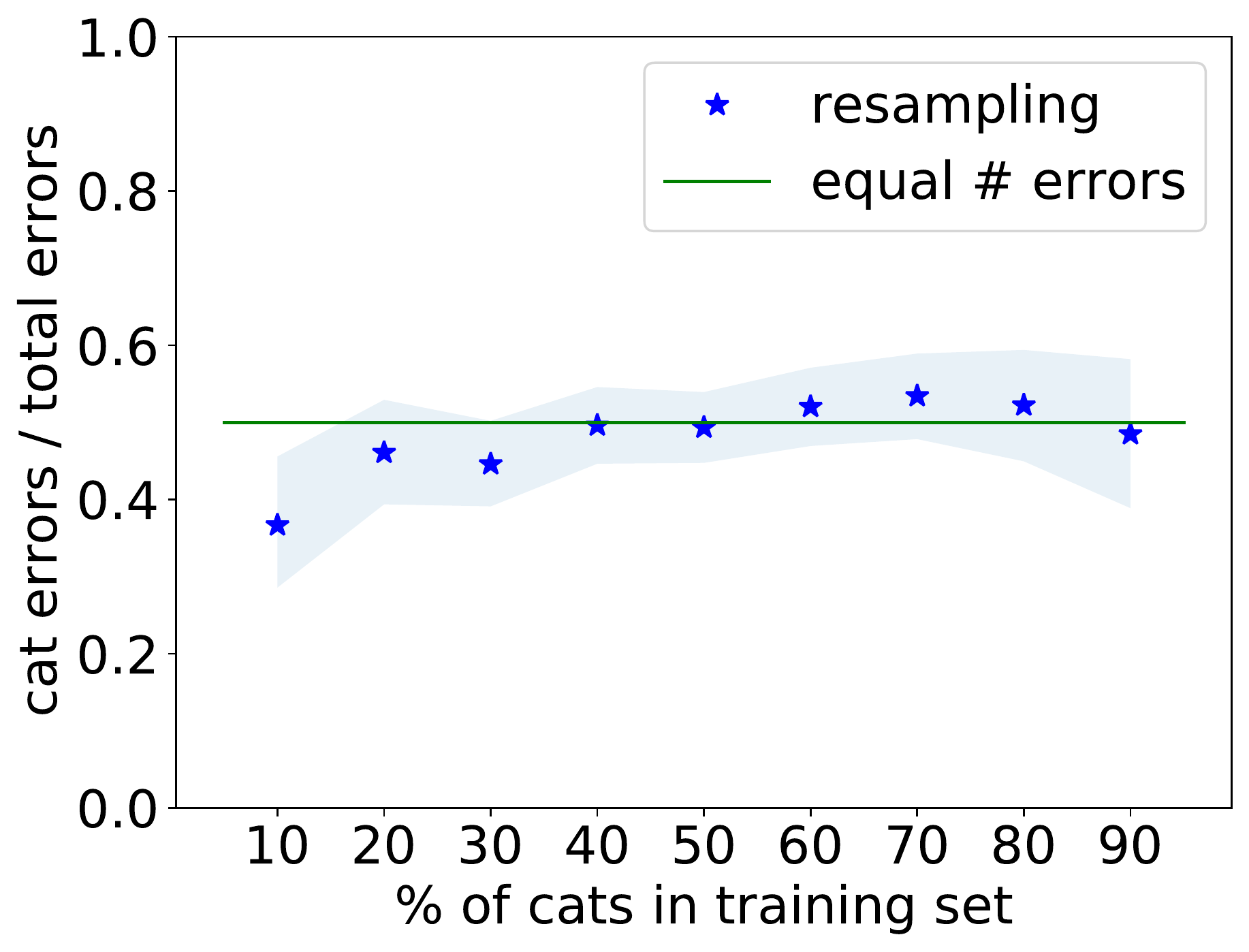}
\end{center}
\caption{We train StyleGAN2 generators of varying bias and test SPE.
  The ground truth images are from the validation set, the observed
  measurements have resolution $4\times 4$. Shaded areas denote 95\%
  confidence intervals. We see that Posterior Sampling satisfies SPE.
  Note that the single failure in the 10\% cat generator is not
  surprising as we are running sequential hypothesis tests on
  non-independent data.}
\label{fig:spe-biases}
\end{figure}

\clearpage

\section{Proofs}\label{app:proofs}
\restate{thm:white-asian}
\begin{proof}
    
 Let $A=A_1 \cup A_2$. We write $p_a = \Pr(x^* \in a)$, $q_{ab} = \Pr(\hat{x} \in b | x^* \in a)$. Using Representation Demographic Parity, with respect first to $\{A,B\}$, then to $\{A_1, A_2, B\}$, we have:
\begin{align*}
    q_{AA} &= q_{BB} \\
    q_{A_1 A_1} &= q_{A_2 A_2} = q_{BB}
\end{align*}
Since $A=A_1 \cup A_2$:
\begin{align*}
    q_{AA} &= \frac{p_{A_1} (q_{A_1 A_1} + q_{A_1 A_2})  + p_{A_2} (q_{A_2 A_1} + q_{A_2 A_2})}{p_{A_1} + p_{A_2}}
\end{align*}
Writing $0 < \frac{p_{A_1}}{p_{A_1} + p_{A_2}} = \alpha < 1$, and replacing $q_{AA}, q_{A_1 A_1}$ and $q_{A_2 A_2}$ by $q_{BB}$, we have:
\begin{align*}
    q_{AA} &= \alpha (q_{A_1 A_1} + q_{A_1 A_2})  + (1- \alpha) (q_{A_2 A_1} + q_{A_2 A_2}) \\
    q_{BB} &= \alpha (q_{BB} + q_{A_1 A_2})  + (1- \alpha) (q_{BB} + q_{A_2 A_2}) \\
    0 &= \alpha q_{A_1 A_2} + (1-\alpha) q_{A_2 A_2}.
\end{align*}
Therefore, an algorithm can satisfy Representation Demographic Parity $\{\{A_1 \cup A_2, B\}, \, \{A_1, A_2, B\}\}$-obliviously if and only if there exists no confusion between $A_1$ and $A_2$, i.e. $q_{A_1 A_2} = 0 = q_{A_2 A_1}$.
\end{proof}

\restate{thm:rdp-general}
  \begin{proof}
     Suppose there exists $x$ such that $\Pr(x) > 0$, and $x_1 \neq x$ such that $\Pr (\hat{x} = x_1) > 0$. Let us split the space into two groups $A$ and $B$, such that both $x$ and $x_1$ belong in $A$. We now further split $A$ into $A_1$ and $A_2$, such that $x_1$ belongs in $A_1$, and $x$ belongs in $A_2$. $A_1$ and $A_2$ now are not perfectly distinguishable, so using the claim above, Representation Demographic Parity is not satisfiable $\{\{A_1 \cup A_2, B\}, \, \{A_1, A_2, B\}\}$-obliviously, so it cannot be satisfiable obliviously. 
  \end{proof} 

\restate{prop:rdp_pr_incompat}
\begin{proof}
Suppose towards a contradiction that both PR and RDP hold, and the distribution is such that $$\Pr(x^* \in c_1) > \frac{1}{2} > \sum_{i \neq 1} \Pr(x^* \in c_i).$$ Since PR holds, $\Pr(\hat{x} \in c_1) = \Pr(x^* \in c_1)$. However, since RDP holds and the algorithm does not reconstruct each class perfectly we have $\alpha = \Pr(\hat{x} \in c_i \mid x^* \in c_i) < 1$ for all the $i$. We now observe the following contradiction. 
\begin{align*}
	\Pr(\hat{x} \in c_1) &\leq \sum_i \Pr(\hat{x} \in c_1 \mid x^* \in c_i) \Pr(x^* \in c_i)  \\
	 &\leq (1-\alpha) \sum_{i \neq 1} \Pr(x^* \in c_i) + \alpha \Pr(x^* \in c_1) \\
    &< (1-\alpha) \Pr(x^* \in c_1) + \alpha \Pr(x^* \in c_1)\\
    & = \Pr(x^* \in c_1). 
\end{align*}

\end{proof} 

\posteriorsamplingonly*
\begin{proof} 
	Let $\cA$ denote a reconstruction algorithm. Given measurements $y$,
	let $Q(U|y)$ denote the probability that the reconstruction from 
	algorithm $\cA$ lies in the measurable set $U$.

	If $\cA$ satisfies CPR, then for
	all measurable $U \subset \R^n,$ and all $y \in \R^m,$ we have
	\begin{align*}
		Q(U|y) = P(U|y).
	\end{align*}

	By the definition of the total variation distance, we have
	\begin{align*}
		TV(Q( \cdot |y), P( \cdot |y)) &= \sup_{U \in \cB(\R^n)} Q(U|y) -
		P(U|y).
	\end{align*}

	Since we have $Q(U|y) = P(U|y)$ for all measurable $U\in \cB(\R^n)$
	and almost all measurements $y \in \R^m,$ we have $TV(Q(\cdot | y),
	P(\cdot|y))=0$ for almost all $y \in \R^m$.

	This shows that the output distribution of $\cA$ must exactly match
	the posterior distribution $P(\cdot | y)$, and hence posterior sampling is the only algorithm that can satisfy obliviousness and CPR. 
\end{proof}

\restate{thm: cpr implies spe}
\begin{proof}
  We want to show that if $\Pr(\wh{x} \in c_i | y) = \Pr(x^* \in c_i |
  y), \forall c_i \in C,$ for almost all $y\in \R^m$, then we have
  $\Pr(\wh{x} \in c_i, x^* \in c_j) = \Pr(\wh{x} \in c_j, x^* \in c_i),
  \forall c_i, c_j \in C$.

  Consider the term $\Pr(\wh{x} \in c_i, x^* \in c_j)$. We can write
  this as an average over $y$, to get:
  \begin{align*}
    \Pr(\wh{x} \in c_i, x^* \in c_j) 
    &= \E_y \Pr( \wh{x} \in c_i, x^* \in c_j | y ).
  \end{align*}
  Note that $\wh{x} \& x^*$ are conditionally independent given $y$.
  This is because $\wh{x}$ is purely a function of $y$. This gives
  \begin{align*}
    \Pr(\wh{x} \in c_i, x^* \in c_j) 
    &= \E_y [ \Pr( \wh{x} \in c_i| y ) \Pr( x^* \in c_j | y )].
  \end{align*}

  If we have CPR with respect to $c_i$ and $c_j$, then we can rewrite
  the above equation as
  \begin{align*}
    \Pr(\wh{x} \in c_i, x^* \in c_j) 
    &= \E_y [\Pr( x^* \in c_i| y ) \Pr( \wh{x} \in c_j | y )].
  \end{align*}
  Using the conditional independence of $x^*, \wh{x}$ given $y$, we
  now have 
  \begin{align*}
    \Pr(\wh{x} \in c_i, x^* \in c_j) 
    &= \E_y \Pr( x^* \in c_i, \wh{x} \in c_j | y ),\\
    &= \Pr( \wh{x} \in c_j, x^* \in c_i ).
  \end{align*}
  This completes the proof.

\end{proof}

\restate{cor:cond_resamp_achieves_pairwise_error}
\begin{proof}
  The proof follows directly from Theorem~\ref{thm:cr_cpr} and
  Theorem~\ref{thm: cpr implies spe}
\end{proof}

\thmcondreweight*

In the special case of 2 classes, the reweighting is very simple: $\lambda_1 = \lambda_2 = \frac{1}{2}$.

\begin{proof} 
We will prove this theorem by contradiction. Before we start, observe that if we scale the mass of $c_i$ by $\lambda_i \geq 0$, we have, 

\begin{align*} 
\alpha_i &:= \Pr(\hat{x} \in C_i \mid x^* \in C_i) \\
&= \E_{y \mid x^* \in C_i} \left[ \frac{ \lambda_i \Pr(x^* \in C_i \mid y) }{\sum_{j} \lambda_j \Pr(x^* \in C_j \mid y) } \right] 
\end{align*}

WLOG, assume $\sum_i \lambda_i = 1$,  this can be done by rescaling the $\lambda_i$'s by their sum. RDP is achieved if all the $\alpha_i$ are equal. Let the smallest $\alpha_i$ when all the $\lambda_i$'s are equal be $\epsilon$. Consider the set $T := \{ \vec{\lambda} \mid \sum_i \lambda_i = 1, \forall i~ \alpha_i  \geq \epsilon \}$. 

Towards a contradiction, suppose no assignment of $\vec{\lambda} \in T$ achieves $f(\vec{\lambda}) :=  \frac{\max_i \alpha_i}{ \min_j \alpha_j} = 1$. Let $r := \min_{\lambda_1, \dots, \lambda_k} f(\vec{\lambda}) > 1$, and let $\vec{\lambda}^*$ be a point which achieves this. $\vec{\lambda}^*$ exists since $f(\vec{\lambda})$ is continuous over $T$, which is compact.

We will show that there exists $\vec{\lambda}' \in T$ such that $f(\lambda') < r$, which contradicts our hypothesis. Let $S := \{ i \in [k] \mid \alpha_i \leq \sqrt{r} \min_i \alpha_i \}$ and, 

$$ \lambda'_i = 
\begin{cases}
      r^{1/4} \lambda^*_i ,& \text{if } i \in S\\
  \lambda_i^*, & \text{otherwise}
\end{cases}
$$

Let $\alpha'_i$ be $\Pr(\hat{x} \in C_i \mid x^* \in C_i)$ where the probability is with respect to the modified distribution. For $i \in S$, 

\begin{align*} 
\alpha'_i &= \E_{y \mid x^* \in C_i} \left[ \frac{ \lambda'_i \Pr(x^* \in C_i \mid y) }{\sum_{j} \lambda'_j \Pr(x^* \in C_j \mid y) } \right] \\
&=r^{1/4} \E_{y \mid x^* \in C_i} \left[ \frac{ \lambda_i \Pr(x^* \in C_i \mid y) }{\sum_{j} \lambda'_j \Pr(x^* \in C_j \mid y) } \right] \\
&\leq r^{1/4} \E_{y \mid x^* \in C_i} \left[ \frac{ \lambda_i \Pr(x^* \in C_i \mid y) }{\sum_{j} \lambda_j \Pr(x^* \in C_j \mid y) } \right]
\end{align*}

Where the last line follows from the fact that $\lambda'_i \geq \lambda_i$ for all $i$, since $r > 1$, which means the denominator only increases. A similar calculation shows that if $i \notin S$, then each $\alpha_i$ is multiplied by a factor between $r^{-1/4}$ and $1$. 

We notice that, by definition of $S$, all the $j$ such that $\alpha_j = \min_i \alpha_i$ are in $S$, and all the $k$ such that $\alpha_k = \max_i \alpha_i$ are not in $S$. This ensures that $\max_i \alpha_i' < \max_i \alpha_i$ and $\min_i \alpha'_i > \min_i \alpha_i \geq \epsilon$. Which, in turn, contradicts the hypothesis that $r$ was the smallest achievable ratio with the original constraints, since we can always renormalize $\lambda'_i$ without affecting the $\alpha_i$. 
\begin{align*}
\frac{\max_i \alpha'_i}{\min_i \alpha'_i} < \frac{\max_i \alpha_i}{\min \{ r^{-1/4} \sqrt{r} \min_i \alpha_i,  \min_i \alpha_i\} } \leq r.        
\end{align*}
\end{proof}

\restate{thm:opt-rce}
\begin{proof}
	The proof follows from Lemma~\ref{lemma:logaccuracy}. Note that
	$H(U|Y)$ is a function of $x^* \& y$ and hence has no dependence on
	the reconstruction algorithm. By the non-negativity of $KL$
	divergence, the representation cross-entropy is minimized when $Q(U_i |
	y) = P(U_i | y)$ for each $i \in [N],$ almost surely over $y$.
\end{proof}

\begin{lemma}
	Let $U : \R^n \to \left\{ c_1, c_2, \cdots , c_k \right\}$ be a function
	that encodes which group contains an image, and assume that the
	groups $c_1, \cdots, c_k \subset \R^n$ are disjoint and form a
	partition of $\R^n$.

	For a reconstruction algorithm $\cA,$ let $Q(c_i|Y)$ denote
	the probability that the reconstruction lies in the set $c_i$ given
	measurements $y$.  Let $P(c_i | y)$ denote the probability that
	$x^*$ lies in $U_i$ conditioned on $y$. 

	Then we have
	\begin{align*}
		RCE(\cA) = H_P\left( U | y \right) + \E_{y}\left[
		KL\left( P(U|y) \| Q(U |y) \right) \right],
	\end{align*}
	where 
	\begin{align*}
		H_P(U|y) & := -\E_{y}\left[ \sum_{i \in [k]} P(c_i |y) \log
	P(c_i|y)\right],\\
	KL(P(U|y) \| Q(U|y)) &:= \sum_{i\in [k]} P(c_i | y) \log \left(
	\frac{P(c_i | y)}{Q(c_i | y)} \right).
	\end{align*}
	
	\label{lemma:logaccuracy}
\end{lemma}
\paragraph{Remark:}
There is a slight abuse of notation in the lemma. Since $U$ is a function of $x^*,$ when treating $x^*$ as a random variable, we also treat $U$ as a random variable.

\begin{proof}
	By the definition of $RCE$ and the tower property of
	expectations, we have
	\begin{align*}
		-RCE(\cA)=&\E_{x^*,y} \log \Pr \left[ \wh{x} \in U( x^*) | y \right] =
		\E_{y} \E_{x^*|y}\left[ \log\left( \Pr\left[ \wh{x} \in U(x^*) | y
		\right]\right)\right],\\
		=& \E_{y} \E_{x^*|y}\left[ \sum_{i\in [N]} \bm{1}\left\{ x^* \in
		U_i \right\}\log\left( \Pr\left[ \wh{x} \in U(x^*) | y\right]
		\right) \right],\\
		=& \E_{y} \E_{x^*|y}\left[ \sum_{i\in [N]} \bm{1}\left\{ x^* \in
		U_i \right\}\log\left( \Pr\left[ \wh{x} \in U_i | y
		\right]\right)\right],\\
		=& \E_{y} \E_{x^*|y}\left[ \sum_{i\in [N]} \bm{1}\left\{ x^* \in
		U_i \right\}\log\left(Q(U_i | y) \right)\right],\\
		=& \E_{y} \left[ \sum_{i \in [N]}P(U_i | y) \log\left(Q(U_i | y)
		\right)\right].  (*)
	\end{align*}
	where the second line follows because the $U_i$s form a partition,
	the third line follows since $\wh{x} \in U(x^*)$ is equivalent
	to $\wh{x} \in U_i$ if we know that $x^* \in U_i$, the fourth line
	follows from the definition of $Q(U_i | y )$ and the last line
	follows from linearity of expectation.

	Now we can multiply and divide $P(U_i|y)$ within the $\log$ term
	above.  This gives 
	\begin{align*}
		(*)=& \E_{y} \left[ \sum_{i\in [N]} P\left( U_i | y \right)
		\log\left(\frac{Q(U_i |y) P(U_i |y)}{P(U_i|y)}\right)\right], \\
		=& \E_{y} \left[ \sum_{i\in [N]} P\left( U_i | y \right)
		\log\left( P(U_i | y)\right)\right] \\
		+& \E_{y} \left[ \sum_{i\in [N]} P\left( U_i | y \right)
		\log\left(\frac{Q(U_i |y) }{P(U_i|y)}\right)\right], \\
		=& - H\left( U | y \right) - \E_{y}\left[ KL\left( P(U|y) \|
		Q(U|y) \right) \right] .
	\end{align*}

	This concludes the proof.

\end{proof}

\section{Langevin Dynamics}\label{app:langevin}
\subsection{StyleGAN2}
We want to sample from the distribution $p(x|y)$ induced by a
StyleGAN2. Note that sampling from the marginal distribution $p(x)$ of
a StyleGAN2 is achieved by sampling a latent variable $z\in\R^{512},$
and 18 noise variables $n_i \in \R^{d_i}$ of varying sizes, and
setting $x = G(z,n_1,\cdots, n_{18}).$ Hence, we can sample from
$p(x|y)$ by sampling $\wh{z}, \wh{n}_1, \cdots, \wh{n}_{18},$ from
$p(z,n_1,\cdots,n_{18}|y),$ and setting $\wh{x} =
G(\wh{z},\wh{n}_1, \cdots,\wh{n}_{18}).$ 

The prior of the latent and noise variables is a standard Gaussian
distribution.  Since we know the prior distribution of these
variables, if we know the distribution of the meaurement process, we
can write out the posterior distribution.  

For the measurement process we consider, we have $y=Ax^*$, where $A$
is a blurring matrix of appropriate dimension. Note that in the
absence of noise, posterior sampling must sample solutions that
exactly satisfy the measurements. However, this is difficult to
enforce in practice, and hence we assume that there is some small
amount of Gaussian noise in the measurements. In this case, the
posterior distribution becomes:
\begin{align}
  p(z,n_1,\cdots,n_{18}|y) & \propto p(y|z,n_1,\cdots,n_{18})
  p(z,n_1,\cdots,n_{18}),\\
  \Leftrightarrow \log p(z,n_1,\cdots,n_{18}|y) &=
  -\frac{\norm{y-AG(z,n_1, \cdots,n_{18})}^2}{2\sigma^2/m} -
  \norm{z}^2/2 - \sum_{i=1}^18 \norm{n_i}^2/2 + c(y),
\end{align}
where $c(y)$ is an additive constant which depends only on $y$.

Now, Langevin dynamics tells us that if we run gradient ascent on the
above log-likelihood, and add noise at each step, then we will sample
from the conditional distribution asymptotically. Please note that we
sample $z$ and \emph{all noise variables} $n_1,\cdots,n_{18}.$

In our experiments, we do 1500 gradient steps. 
In practice, we replace the $\sigma$ in the equation above with
$\sigma_t$, where $t$ is the iteration number. When the measurements
have resolution $8\times 8$ or $4\times 4$, we find that
$\sigma_1=1.0,\sigma_{1500}=0.1$ works best. When the resolution of the
measurements is $2\times 2$ or $1\times 1$, we find that
$\sigma_1=1.0,\sigma_{1500}=0.01$ works best. We
change the value of $\sigma_t$ after every 3 gradient steps, such that
$\sigma_1, \sigma_4, \sigma_7,\cdots,\sigma_{1497}$ form a geometrically
decreasing sequence. The learning rate $\gamma_{1500}$ is also tuned to be
a decreasing geometric sequence, such that $\gamma_t=5\cdot 10^{-6}.$
Please see~\cite{song2019generative} for the equations specifying the
learning rate tuning, and the logic behind it.

We also find that adding a small amount of noise corresponding to
$\sigma_{1500}$ in the measurements helps Langevin mix better.

We note that our approach is different from prior
work~\cite{karras2020analyzing,menon2020pulse}, which optimizes a
function of our variable $z$, and a subset of the noise variables.

\paragraph{NCSNv2}
The NCSNv2 model~\cite{song2020improved} has been designed such that
sampling from the marginal distribution requires Langevin dynamics.
This model is given by a function $s(x;\sigma)$, which outputs $\nabla
\log p_\sigma(x),$ where $p_\sigma(x)$ is the distribution obtained by
convolving the distribution $p(x)$ with Gaussian noise of variance
$\sigma^2$. That is, $p_\sigma (x) = (p* \cN(0,\sigma^2))(x)$.

It is easy to adapt the NCSNv2 model to sample from posterior
distributions, see~\cite{song2020improved} for inpainting examples. In
our super-resolution experiments, one can compute the gradient of
$p(y|x)$, to get the following update rule for Langevin dynamics:
\begin{align}
  x_{t+1} \leftarrow x_{t} + \gamma_t ( s(x_t; \sigma_t) - A^T(Ax_t -
  y)/\sigma_t^2 ) + \sqrt{2\gamma_t} \xi_t,
\end{align}
where $A$ is the blurring matrix, and $\xi_t \sim \cN(0,I_n)$ is
i.i.d. Gaussian noise sampled at each step. We use the default values
of noise and learning rate specified in
\url{https://github.com/ermongroup/ncsnv2/blob/master/configs/ffhq.yml}.
That is, $\sigma_1=348, \sigma_{6933}=0.01$, and $\gamma_{6933}=9\cdot
10^{-7}.$ Note that the value of $\sigma,\gamma$ changes every 3
iterations, and both $\gamma_t$ and $\sigma_t$ decay geometrically.
See~\cite{song2020improved} for specific details on how these are
tuned.

\section{Code}
All code and generative models, along with hyperparameters and README
are available at
\url{https://github.com/ajiljalal/code-cs-fairness}.

\end{document}

% This document was modified from the file originally made available by
% Pat Langley and Andrea Danyluk for ICML-2K. This version was created
% by Iain Murray in 2018, and modified by Alexandre Bouchard in
% 2019 and 2021. Previous contributors include Dan Roy, Lise Getoor and Tobias
% Scheffer, which was slightly modified from the 2010 version by
% Thorsten Joachims & Johannes Fuernkranz, slightly modified from the
% 2009 version by Kiri Wagstaff and Sam Roweis's 2008 version, which is
% slightly modified from Prasad Tadepalli's 2007 version which is a
% lightly changed version of the previous year's version by Andrew
% Moore, which was in turn edited from those of Kristian Kersting and
% Codrina Lauth. Alex Smola contributed to the algorithmic style files.